\documentclass[journal]{IEEEtran}

\ifCLASSINFOpdf
\usepackage[pdftex]{graphicx}
\graphicspath{{../pdf/}{../jpeg/}}
\DeclareGraphicsExtensions{.pdf,.jpeg,.png}
\else
\usepackage[dvips]{graphicx}
\graphicspath{{../eps/}}
\fi
\usepackage{graphicx}
\usepackage{epstopdf}
\usepackage{diagbox}
\usepackage{multirow}
\usepackage[cmex10]{amsmath}
\usepackage[ruled]{algorithm2e} 
\usepackage{url}
\usepackage{lineno,hyperref}
\usepackage{amsthm}
\usepackage{booktabs}
\usepackage{subfigure}
\usepackage{url}
\usepackage{mathrsfs}
\usepackage{amsmath}
\usepackage{amsthm}
\usepackage{bbding}
\usepackage{xcolor}
\usepackage{amsfonts,amssymb}
\usepackage{mathrsfs}
\usepackage{amsmath}
\usepackage{amsthm}
\usepackage{bbding}
\usepackage{cite}
\usepackage{latexsym,amsmath,xcolor,multicol,booktabs,calligra}

\begin{document}
	\title{Uncertainty Quantification for Incomplete Multi-View Data Using Divergence Measures}
	\author{Zhipeng Xue$^{\dagger}$,
		Yan Zhang$^{\dagger}$, 
		Ming Li$^{\dagger}$,
        Chun Li$^{*}$, ~\IEEEmembership{Member,~IEEE}, Yue Liu, and Fei Yu,~\IEEEmembership{Fellow,~IEEE}
		\thanks{This research was supported by Guangdong Basic and Applied Basic Research Foundation (No.2024A1515011774).} 
		\thanks{Z. Xue, Y. Zhang and C. Li are with MSU-BIT-SMBU Joint Research Center of Applied Mathematics, Shenzhen MSU-BIT University, Shenzhen, 518172, China. E-mail: xue19991219@gmail.com; za1234yuuy@gmail.com} 
        \thanks{Z. Xue and Y. Liu are with School of Optics and Photonics, Beijing Institute of Technology, Beijing 100081, China. E-mail: liuyue@bit.edu.cn}
        \thanks{M. Li and F. Yu are with Guangdong Laboratory of Artificial Intelligence and Digital Economy (SZ), Shenzhen, 518083, China. E-mail: ming.li@u.nus.edu.}
		\thanks{$^{\dagger}$Co-first authors. *Corresponding author: Chun Li (E-mail: lichun2020@smbu.edu.cn).}
	}
	\markboth{IEEE Transactions on Image Processing}%
	{Shell \MakeLowercase{\textit{et al.}}: Bare Demo of IEEEtran.cls for IEEE Journals}
	\maketitle

	\begin{abstract}
                Existing multi-view classification and clustering methods typically improve task accuracy by leveraging and fusing information from different views. However, ensuring the reliability of multi-view integration and final decisions is crucial, particularly when dealing with noisy or corrupted data. Current methods often rely on Kullback-Leibler (KL) divergence to estimate uncertainty of network predictions, ignoring domain gaps between different modalities. To address this issue, KPHD-Net, based on Hölder divergence, is proposed for multi-view classification and clustering tasks. Generally, our KPHD-Net employs a variational Dirichlet distribution to represent class probability distributions, models evidences from different views, and then integrates it with Dempster-Shafer evidence theory (DST) to improve uncertainty estimation effects. Our theoretical analysis demonstrates that Proper Hölder divergence offers a more effective measure of distribution discrepancies, ensuring enhanced performance in multi-view learning. Moreover, Dempster-Shafer evidence theory, recognized for its superior performance in multi-view fusion tasks, is introduced and combined with the Kalman filter to provide future state estimations. This integration further enhances the reliability of the final fusion results. Extensive experiments show that the proposed KPHD-Net outperforms the current state-of-the-art methods in both classification and clustering tasks regarding accuracy, robustness, and reliability, with theoretical guarantees.
	\end{abstract}
	
	
	\IEEEpeerreviewmaketitle

\section{Introduction}
Multi-view learning has become a significant paradigm in machine learning, aiming to integrate information from multiple perspectives or data representations \cite{p18}. By leveraging the complementary information contained in different views, this approach has shown great potential in enhancing the performance of clustering task \cite{p4}. Despite the success of traditional multi-view learning methods, they often struggle with effectively integrating diverse sources of evidence, especially when the quality and relevance of these sources vary dynamically. These methods tend to underestimate uncertainty\cite{p87}, typically assuming equal value for different views or learning a fixed weight factor for each view, which can lead to unreliable predictions, particularly when noisy or corrupted views are present (e.g., information from faulty sensors) \cite{p14,p15}. Consequently, an algorithm is essential that dynamically assigns weights to each modality, accounting for uncertainty to prevent unreliable predictions \cite{p17}. 

To address this issue, enhanced trusted multi-view classification (ETMC) \cite{p18} replaces the Softmax function with a non-negative activation function and uses subjective logic theory \cite{p82} and Dempster-Shafer evidence theory \cite{p26} for reliable uncertainty estimation, achieving excellent results.

In this work, KPHD-Net, based on Hölder divergence \cite{p19} and Kalman filtering \cite{p30}, is proposed. Compared to traditional methods, KPHD-Net integrates the non-negative activation function from ETMC, subjective logic theory, and Dempster-Shafer evidence theory to ensure reliable uncertainty estimation. Additionally, the original KLD \cite{p20} is replaced by Proper Hölder divergence (PHD) to enhance classification accuracy. A comprehensive mathematical justification for PHD's superiority over KLD in multi-view learning is provided. Furthermore, Kalman filtering is combined with Dempster-Shafer evidence theory to yield more robust uncertainty measures. The approach is validated on eight multi-view datasets—four for classification and five for clustering—demonstrating superior performance compared to existing models.

In summary, our work offers the following contributions: 1. The KPHD-Net is developed based on Proper Hölder Divergence, providing a mathematical explanation of PHD's advantages over traditional KLD in multi-view learning. 2. The KPHD-Net is evaluated on two tasks: supervised classification and unsupervised clustering. Nine multi-view datasets are divided into two groups, with four used for classification and five for clustering. Furthermore, noise is introduced to the datasets to test the KPHD-Net's robustness, and the results demonstrate the superiority of the proposed approach. 3. Extensive experiments, including noise addition and ablation studies, are conducted to validate the effectiveness of the KPHD-Net in classification and clustering tasks. This exploration offers new insights into enhancing multi-view classification and clustering models, confirming that the improved objective function significantly boosts classification and clustering efficiency.

\section{Related Work}
\subsection{Multi-view Learning}
Deep learning has been widely applied in various scenarios such as intelligent transportation \cite{li2021self,li2021exploiting,Yan_2025_CVPR}, cross-modal learning \cite{Liu_2025_CVPR, zhao2025favchat, li2025uni}, privacy protection \cite{li2023dr,li2023stprivacy}, and artificial intelligence generated content \cite{li2024instant3d, liu2024realera,zhuang2025vistorybench}.
Similarly, multi-view representation learning has been extensively studied in recent years, with methods broadly categorized into consensus-based and complementary-based approaches \cite{p62}. Consensus-based methods aim to find a common representation that captures the shared information across all views, thereby reducing the complexity of dealing with multiple representations \cite{p63}. Notable examples include Canonical Correlation Analysis (CCA) \cite{p22} and Co-training. Canonical Correlation Analysis (CCA) is a classical method that seeks linear projections of the data in each view such that the correlations between the projected views are maximized. This method has been widely used due to its simplicity and effectiveness in capturing the shared structure among multiple views. Co-training, another influential method, trains separate models on each view and uses the confident predictions from one view to augment the training data for the other view \cite{p25}. This iterative process continues until the models converge, effectively leveraging the consensus between the views to improve learning performance.

\subsection{Dempster-Shafer Evidence Theory (DST)}
DST \cite{p26}, also known as the DST of Evidence, provided a framework for modeling and combining uncertain information. It extends traditional probabilistic approaches by allowing for the representation of partial or uncertain knowledge and combining evidence from multiple sources in a more flexible manner. DST allows beliefs from different sources to be combined by various fusion operators to obtain a new belief that considers all available evidence \cite{p28}. Moreover, Luo et al. \cite{p64} explored the application of variational quantum linear solvers to enhance combination rules within the framework of DST.

\subsection{Hölder Divergence}

The term “Hölder divergence” was derived from Hölder's inequality \cite{p29}. This concept incorporates the idea of inequality tightness and includes two main types: Hölder Statistical Pseudo-Divergence (HPD) and Proper Hölder Divergence (PHD) \cite{p19}. In this work, we focus primarily on PHD, which is defined as follows:
\newtheorem{definition}{\bf{Definition}}
\begin{definition}
    \textbf{(Proper Hölder Divergence, PHD) \cite{p19}} For conjugate exponents \(\alpha, \beta > 0\) and \(\gamma > 0\), the appropriate Hölder divergence (HD) between two density functions \(p(x)\) and \(q(x)\) is defined as follows:
    \begin{equation}
        D_{\alpha,\gamma}^{\mathrm{H}}(p:q)= -\log\left(\frac{\int_{\mathcal{X}}p(x)^{\gamma/\alpha}q(x)^{\gamma/\beta}\mathrm{d}x}
    {(\int_{\mathcal{X}}p(x)^{\gamma}\mathrm{d}x)^{1/\alpha}(\int_{\mathcal{X}}q(x)^{\gamma}\mathrm{d}x)^{1/\beta}}\right).
    \end{equation}
\end{definition}

Due to the several advantages of Hölder divergence over KLD, PHD has been increasingly applied to various machine learning tasks in recent years. For example, in 2015, Hoang et al. \cite{p66} explored the application of Cauchy-Schwarz divergence—a special case of Hölder divergence—in analyzing and comparing Poisson point processes, offering valuable insights into its effectiveness for statistical and machine learning tasks involving these processes. In 2017, Frank et al. \cite{p67} investigated K-means clustering using Hölder divergences, enhancing clustering performance by leveraging the strengths of Hölder divergence over traditional distance measures. Additionally, Pan et al. \cite{p68} utilized Hölder divergence for classification in PolSAR image analysis.
\renewcommand\arraystretch{1.0}
\begin{table}[t]
	\setlength{\belowdisplayskip}{0pt}
	\setlength{\abovedisplayskip}{0pt}
	\setlength{\abovecaptionskip}{0pt}
	\centering
	\scriptsize
	\caption{Main Notations Used in This Work. This table provides a summary of all important symbols and their corresponding definitions as used throughout the manuscript.}
	\setlength{\tabcolsep}{2pt}
	\begin{tabular}{p{3cm}p{5cm}}  
		\toprule [1.0pt]
		Notation&Definition\\
		\midrule[0.5pt]
		$\alpha, \beta,\gamma$ &  The conjugate exponents of Hölder \\
		$D_{\alpha,\gamma}^{\mathrm{H}}(p:q)$ & The Proper Hölder Divergence of $p(x)$ and $q(x)$\\
		$b_k^i$ & Reliability of the $kth$ classification result for the $ith$ modality \\
		${{\rm M}^i} = \left\{ {\{ b_k^i\} _{k = 1}^K,{u^i}} \right\}$&Reliability of the classification result for the ith modality and overall uncertainty\\
		$\left\{ {x_n^m} \right\}_{m = 1}^M,{y_n}$& The $n$ samples with $M$ modalities each, and the labels corresponding to the $n$ samples, respectively\\
		${\lambda _t}, Dir(.|.)$ & Weight parameter and Dirichlet distribution, respectively\\
		\bottomrule[1.0pt]
	\end{tabular}
	\label{tab01}
\end{table}
\subsection{Kalman Filtering}
Kalman filtering is a recursive algorithm used to estimate the state of a linear dynamic system from noisy observations \cite{p30}. It excels at handling Gaussian noise and provides optimal state estimates by combining predictions from a system model with sensor measurements \cite{p31}. For example, Yang et al. \cite{p70} introduced a highly accurate manipulator calibration method that integrates an extended Kalman filter with a residual neural network, enhancing calibration precision and efficiency. Qian et al. \cite{p71} explored observing periodic systems by bridging centralized Kalman filtering with consensus-based distributed filtering, offering a unified approach to improve filtering performance in distributed networks. Su et al. \cite{p72} developed a graph-frequency domain Kalman filtering method designed to manage measurement outliers in industrial pipe networks, improving the accuracy and reliability of monitoring and control systems.

DST, on the other hand, provides a flexible framework for modeling and combining uncertain evidence from multiple sources, accommodating both complete and incomplete information. Integrating Kalman filtering with DS Theory enhances the system's ability to manage non-Gaussian noise and integrate conflicting evidence. This integration leverages Kalman's efficiency in dynamic state estimation and DS's robustness in uncertainty management, resulting in improved accuracy and reliability in environments with dynamic and uncertain data, such as multi-sensor systems.

\begin{figure*}
    \setlength{\belowcaptionskip}{0pt} 
    \setlength{\abovecaptionskip}{0pt} 
    \centering
    \includegraphics[width=0.9\textwidth]{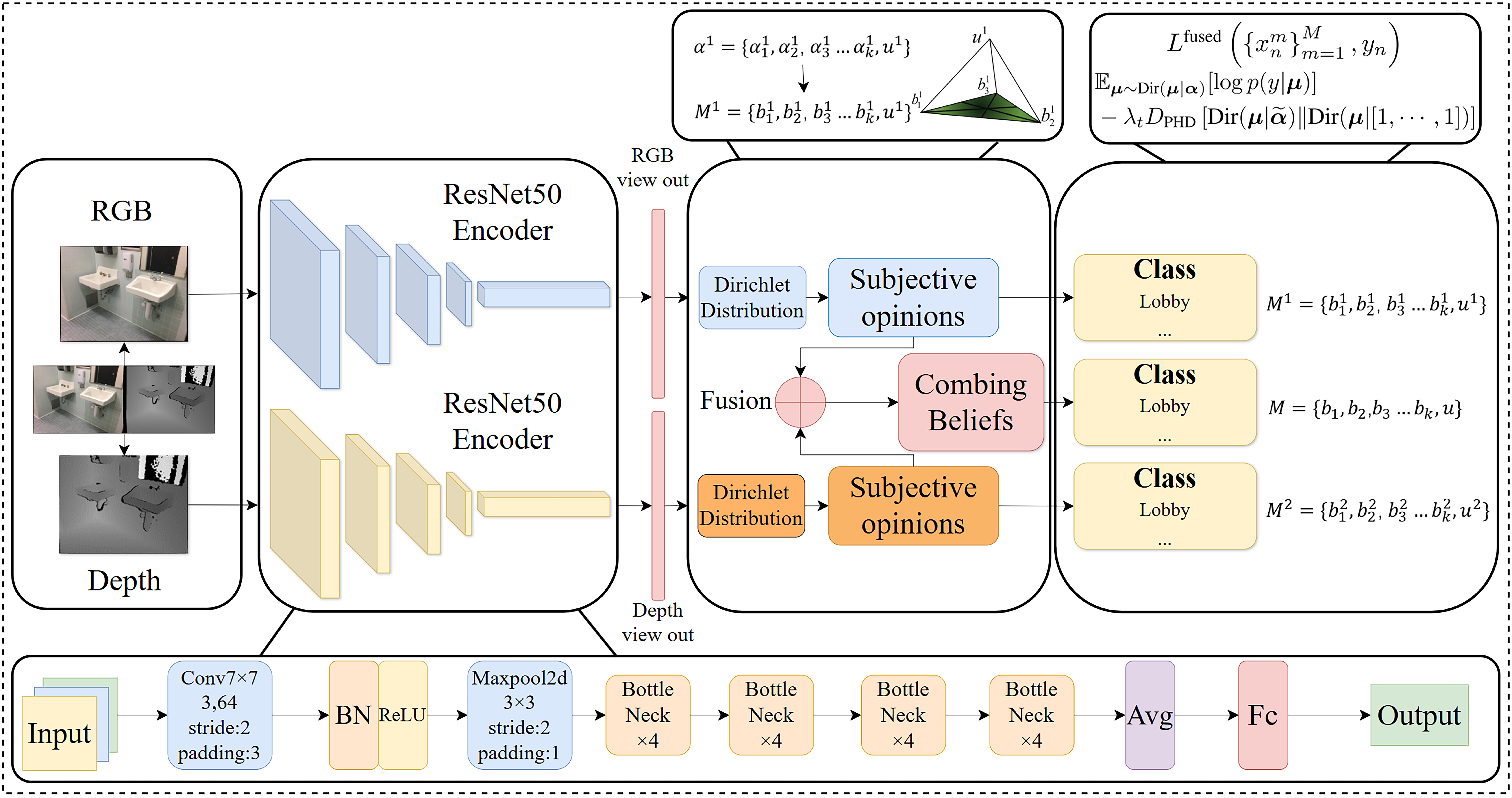}
    \caption{An overview of uncertainty quantification for incomplete multi-view data using proper Hölder divergence measures. This approach highlights the effectiveness of Hölder divergence in capturing the variability and uncertainty inherent in incomplete datasets, enhancing the reliability of multi-view analysis.}
    \label{fig:01}
\end{figure*}

\section{Methodology}
\subsection{Variational Dirichlet Distribution}
The main notations used in this work are presented in Table \ref{tab01}. In machine learning, the representation of compositional data is crucial for addressing multi-class classification problems. Aitchison \cite{p32} introduced the Dirichlet distribution as a primary model for such data. For a classification problem involving $K$ classes, the objective is to develop a function that can generate predicted class labels while minimizing the discrepancy between these predictions and the true labels. Traditionally, deep learning approaches use the softmax operator to convert continuous model outputs into class probabilities. However, this operator often results in overconfident predictions \cite{p18}.

The Dirichlet distribution serves as a versatile and essential tool in probabilistic modeling, particularly in the context of multi-class classification and Bayesian inference. It is the conjugate prior for the multinomial distribution, which simplifies Bayesian inference by ensuring that the posterior distribution retains the same form as the prior \cite{p33}. This property not only facilitates the analytical tractability of Bayesian updates but also supports various applications, including Bayesian statistics, natural language processing, and computer vision. The Dirichlet distribution's key advantages include its flexibility in modeling categorical data, parameter interpretability, smoothing capabilities, and its suitability for generative and hierarchical modeling tasks \cite{p33}.
\begin{algorithm}[t]
	\caption{\small Uncertainty Quantification for Incomplete
Multi-View Data Using Divergence Measures.}
	\label{alg:spl}
	\DontPrintSemicolon
	\small
	\tcp*[f]{\textbf{*Training*}}\\
	\textbf{Input:} Multi-View Dataset: $D = \left\{ {\left\{ {{\rm X}_n^m} \right\}_{m = 1}^M,{y_n}} \right\}_{n = 1}^N$;\\
	\textbf{initialization:} Initialize the parameters of the neural network.\\
	\While{not converged}  
	{ 
		\For {$m=1:M$}{
			(1) $Dir({\mu ^m}|{x^m}) \leftarrow$ variational network output;\\ 
			(2) Subjective opinion ${M^m} \leftarrow Dir({\mu ^m}|{x^m});$}\
		
		(1) Obtain joint opinion ${M^m}$;\\
		(2) Obtain $Dir({\mu ^m}|{x^m})$;\\
		(3) Obtain the overall loss by updating $\alpha$,$\gamma$ and $\left\{ {{\alpha ^v}} \right\}_{v = 1}^V$,$\left\{ {{\gamma ^v}} \right\}_{v = 1}^V$;\\
		(4) Maximize \textbf{objective function} and update the networks with gradient descent;\\
	}
	{\bfseries Output:} networks parameters.\\
	\tcp*[f]{\textbf{*Test*}}\\
	Calculate the joint belief and the uncertainty masses.\\
\end{algorithm}
In multi-view learning, Dirichlet learning offers unique advantages by modeling dependencies between different data views through a stochastic process. It can manage variable-dimensional feature spaces and effectively incorporate prior knowledge, enhancing both classification performance and interpretability \cite{p33}.

To address the issue of overconfidence inherent in the softmax function, a Dirichlet distribution is utilized as a conjugate prior for the multinomial distribution, thereby establishing a predictive distribution. The class probabilities, denoted as $\boldsymbol{\mu}=[\mu_1,\cdots,\mu_K]$, can be interpreted as parameters of a multinomial distribution, where the sum of probabilities is one $\sum_{k=1}^{K}\mu_{k}=1$. This setup characterizes the probability of $K$ mutually exclusive events \cite{p34}. By introducing the Dirichlet distribution, the uncertainty in the predictive distribution can be captured, effectively mitigating the issue of overconfidence. The formulation of the Dirichlet distribution is further expounded upon, starting with the definition of the exponential family to which this distribution is inherently linked. The density function of the Dirichlet distribution is provided in the definition \ref{def_2}.
\begin{definition} 
	\label{def_2}
	(\textbf{Dirichlet Distribution \cite{p33}}) The Dirichlet distribution of order $K$ (where $K\ge 2$) with parameters $\alpha_i>0, i=1,2,3...,K$ is defined by a probability density function with respect to Lebesgue measure on the Euclidean space $R^{K-1}$ as follows: ${\rm{Dirichle}}{{\rm{t}}_n}({\mu _1}, \cdots ,{\mu _K}|{\alpha _1}, \ldots ,{\alpha _K}) = \frac{{\Gamma \left( {\sum\limits_{i = 1}^n {{\alpha _i}} } \right)}}{{\prod\limits_{i = 1}^n \Gamma  ({\alpha _i})}}\prod\limits_{i = 1}^n {\mu _i^{{\alpha _i} - 1}},$ where ${\mu_i} \in {S_K}$, and ${S_K}$ is the standard $K-1$ dimentional simplex, namely, $${{\cal S}_K} = \left\{ {\left( {{\mu _1},{\mu _2},...,{\mu _K}} \right)\mid \sum\limits_{i = 1}^K {{\mu _i}}  = 1,\;0 \le {\mu _1}, \ldots ,{\mu _K} \le 1} \right\},$$ and $\Gamma(.)$ is the gamma function,  defined as: $\Gamma(s)=\int_0^\infty x^{s-1}\mathrm{e}^{-x}\mathrm{~d}x,\quad s>0$.
\end{definition}
\begin{definition} 
	\label{def_5}
	(\textbf{Exponential Family Distribution \cite{p19}}) The probability density function of the Dirichlet distribution is expressed as follows: $p\left( {x;\theta } \right) = \exp \{ {\theta ^ \top }T(x) - F(\theta ) + B(x)\},$ where $\theta$ is the natural parameter, $T(x)$ is the sufficient statistic, $F(\theta)$ is the log-normalizer, and $B(x)$ is the base measure.
\end{definition}

Therefore, based on Definitions \ref{def_2} and \ref{def_5}, the Dirichlet distribution can be expressed in the form of an exponential family distribution as described in Definition \ref{def_6}.
\begin{definition} 
	\label{def_6}
	(\textbf{The Exponential form of the Dirichlet Distribution \cite{p83}}) Exponential formulation of the Dirichlet distribution probability density function can be rewrite as follows: 
	\begin{equation}
		\label{equ_12}
		\begin{array}{l}
			{\rm{Dirichle}}{{\rm{t}}_n}({\mu _1}, \cdots ,{\mu _K}|{\alpha _1}, \ldots ,{\alpha _K})\\
			\begin{array}{*{20}{c}}
				{}
			\end{array} = \exp \left\{ {\sum\limits_i^K {\left( {{\alpha _i} - 1} \right)} \log {\mu _{\rm{i}}} - \left[ {\begin{array}{*{20}{c}}
						{\sum\limits_i^K {\log \Gamma \left( {{\alpha _i}} \right)} }\\
						{ - \log \Gamma \left( {\sum\limits_i^K {{\alpha _i}} } \right)}
				\end{array}} \right]} \right\}.
		\end{array}
	\end{equation}
\end{definition}
Allowing us to obtain the canonical form terms: ${\nabla _\theta }T(\theta ) = \left[ {\begin{array}{*{20}{c}}
		{\psi ({\alpha _1}) - \psi (\sum\limits_{i = 1}^K {{\alpha _i}} )}\\
		\vdots \\
		{\psi ({\alpha _K}) - \psi (\sum\limits_{i = 1}^K {{\alpha _i}} )}
\end{array}} \right],$ $\theta=\boldsymbol{\alpha}$, $T(\boldsymbol{\mu})=ln(\boldsymbol{\mu})$, $F(\eta ) = \sum\limits_{i = 1}^K {\ln } \Gamma ({\alpha _i}) - \ln \Gamma (\sum\limits_{i = 1}^K {{\alpha _i}} )$, $B(\boldsymbol{\mu})=-ln(\boldsymbol{\mu})$, and $\psi$ is the digamma function, defined as: $\psi(x)=\frac{\mathrm{d}}{\mathrm{d}x}\ln\Gamma(x)$.

\subsection{Dempster-Shafer Evidence Theory Combined with Kalman Filtering for Multi-View Learning}
In the ETMC algorithm, modality fusion is primarily based on subjective logic \cite{p37} and DST \cite{p38,p39}. During training, it is essential to quantitatively assess the uncertainty and credibility of each modality, generating specific numerical values\cite{p88}. A simplified evidence theory is then applied to perform the modality fusion. Additionally, subjective logic \cite{p37} is used to evaluate the uncertainty and credibility of classification results from the fused modalities. 

To assess the uncertainty and credibility of each modality, the algorithm employs a Dirichlet distribution, which acts as a distribution over the features derived from the KPHD-Net's classification layer. The confidence in classification outcomes and the quantification of uncertainty are determined through the Dirichlet distribution and subjective logic. Based on these metrics, the algorithm selectively performs modality fusion using evidence theory. 

Furthermore, to generate a Dirichlet distribution, the algorithm replaces the traditional softmax layer with a non-negative activation function layer. The specific steps are as follows:

In a $K$-class classification task, each sample comprises data from $V$ different modalities. For the modality $\mathrm{M}^{1} = \left\{\{b_{k}^{1}\}_{k=1}^{K},u^{1}\right\}$, the uncertainty associated with the confidence in the corresponding classification result can be computed using the Dirichlet distribution. For $\mathrm{M}^{2} = \left\{\{b_{k}^{2}\}_{k=1}^{K},u^{2}\right\}$, the fusion of the modalities, $\stackrel{}{\mathrm{M}}=\mathrm{M}^1\oplus\mathrm{M}^2$, is calculated using the simplified evidence theory. The simplified fusion rules are defined as $b_{k}=\frac{1}{1-C}(b_{k}^{1}b_{k}^{2}+b_{k}^{1}u^{2}+b_{k}^{2}u^{1})$ and $u=\frac{1}{1-C}u^{1}u^{2}$. In this context, each sample includes data from $V$modalities, resulting in $\mathrm{M}=\mathrm{M}^{1}\oplus\mathrm{M}^{2}\oplus\cdots\oplus M^{V}$.

In this study, a method is proposed that combines the Kalman filter with DST to fuse uncertain information from multimodal data. Initially, the \texttt{ds\_bpa} function is used to calculate the basic probability assignment (BPA) for each modality's data and belief. These BPAs include “true”, “false”, and “uncertain” components, representing the likelihood of the classification result being true, false, and uncertain, respectively. To address the uncertainty in sensor data, we utilize the DST to model belief distributions. Specifically, the \texttt{ds\_bpa} function calculates the basic probability assignment (BPA), which quantifies the belief that a particular hypothesis (true or false) is valid based on the available sensor data. The BPA is computed as follows: $\text{bpa(true)} = \text{belief} \times \text{sensor data}$, $\text{bpa(false)} = \text{belief} \times (1 - \text{sensor data}),$ $\text{bpa(uncertain)} = 1 - \text{belief}.$

This function plays a crucial role in incorporating uncertainty into the model by adjusting belief values based on sensor readings.

Subsequently, the \texttt{ds\_combine} function is employed to fuse BPAs from different modalities. This fusion process is based on DS theory's combination rules, effectively integrating information from various modalities to yield a combined belief. The fusion results of BPAs for each modality are calculated through simplified DST rules to produce the final comprehensive belief.

To further enhance the accuracy of the fusion results, a Kalman filter is incorporated. The “true” value from the fused BPAs is used as the observation input \(z\) and processed by the Kalman filter. Specifically, the Kalman filter first predicts the system's state through the \texttt{predict} step, and then updates the state estimate using the \texttt{update} step with the observation input. This process not only smooths the signal but also provides estimates of future states, making the final fused result more reliable and accurate.

The advantage of this approach lies in its combination of DS theory's ability to handle uncertainty and the Kalman filter's dynamic state estimation capability. This combination not only effectively fuses multimodal information and manages uncertainty but also further optimizes the precision of the fusion results through the prediction and update mechanisms of the Kalman filter. Therefore, the method proposed in this study offers an effective solution for multimodal data fusion tasks, particularly in scenarios requiring the handling of significant amounts of uncertain information.

\subsection{Variational Inference for Proper Hölder Divergence-Based Dirichlet Distributions}
A generative model can be expressed as $ p_\theta(x, z) = p_\theta(x|z) p(z) $, where $p_\theta(x|z) $ is the likelihood, and $ p(z) $ is the prior. From the perspective of a Variational Autoencoder (VAE), the true posterior $ p(z|x) $ can be approximated by $ q_\phi(z|x) $. The evidence lower bound (ELBO) for VAE can be formulated as:
\begin{equation}
\label{equ_100}
    \mathcal{L}_{ELBO}(\theta, \phi; x) = \mathbb{E}_{q_\phi(z|x)} [\log p_\theta(x|z)] - D_{KL}(q_\phi(z|x) \| p(z)).
\end{equation}

According to the Cauchy–Schwarz regularized autoencoder \cite{p65}, the objective function with PHD regularization is given by:
\begin{equation}
\label{equ_102}
    \mathcal{L}_{HDR}(\theta, \phi; x) = \mathbb{E}_{q_\phi(z|x)} [\log p_\theta(x|z)] - \lambda D_{\alpha, \gamma}^{H}(q_\phi(z|x) \| p(z)),
\end{equation}
where $ D_{\alpha, \gamma}^{H} $ represents the PHD, and $\lambda $ is the regularization parameter.

\subsection{Key Steps in Evidence Fusion}
The multi-view classification and clustering capabilities of  KPHD-Net stem from its ability to process multiple data modalities and synthesize them into a coherent representation. Our approach utilizes the Dirichlet distribution to represent the uncertainty and confidence of predictions, and the process involves several key steps:

\begin{enumerate} 

\item \textbf{Feature Extraction and Fusion}: KPHD-Net extracts features from each modality using a set of encoders, which are then fused into a unified latent space. This is achieved through a carefully designed encoder network capable of handling diverse modalities, such as RGB images and depth information. Indeed, the PHD is not limited to multi-view scenarios, and we have considered its performance in single-view settings as well. In our experimental results, both the RGB and Depth views represent single-view scenarios.

\item \textbf{Evidence Collection and Dirichlet Parameterization}: KPHD-Net generates Dirichlet parameters for each modality, serving as evidence for each cluster. These parameters are then combined across modalities to produce a fused Dirichlet distribution, encapsulating the collective evidence from all data sources.

\item \textbf{Uncertainty-Aware for Classification and Clustering}: The parameters of the Dirichlet distribution provide a nuanced representation of data points, capturing both their cluster assignments and associated uncertainties. This approach is particularly beneficial in scenarios involving ambiguous data, as it prevents premature commitment to uncertain cluster assignments.

\item \textbf{Kalman Filtering for Dynamic Adjustment}: A Kalman filter is employed to dynamically adjust the fused evidence, ensuring that the KPHD-Net's clustering decisions are continuously refined based on incoming data. This integration helps maintain robustness and adaptability in the face of varying data quality or shifting distributions.
\end{enumerate}

This method not only achieves high clustering accuracy but also provides a measure of certainty for each clustering decision, thereby enhancing the reliability of the clustering outcomes. The KPHD-Net's capability to perform both classification and clustering tasks showcases its versatility and robustness in handling multi-view data. By employing Dirichlet distributions for uncertainty modeling, along with evidence theory and Kalman filtering, KPHD-Net offers a sophisticated approach to multi-view learning, delivering both high accuracy and reliable uncertainty estimates.

\subsection{Loss Function}

The cross-entropy loss function is commonly used in classification tasks and is defined as:
\begin{equation}
    \mathcal{L}(y, \hat{y}) = -\sum_{i=1}^{N} y_i \log(\hat{y}_i),
\end{equation}
where $y_i$ is the true label, typically represented as a one-hot encoded vector, $\hat{y}_i$ is the predicted probability for class $i$ given by the model, and $N$ is the total number of classes. 

The ETMC algorithm enhances classification network generalization and robustness by optimizing a loss function that accounts for both classification outcomes and associated confidence and uncertainty. After that, we can get the loss function as follows:
\begin{equation}
    \begin{aligned}\mathcal{L}^{overall}&=\sum_{i=1}^{N}\mathcal{L}^{fused}\left(\left\{x_{n}^{m}\right\}_{m=1}^{M},y_{n}\right)+\sum_{i=1}^{N}\sum_{m=1}^{M}\mathcal{L}^{m}\left(x_{n}^{m},y_{n}\right)\\&+\sum_{i=1}^{N}\mathcal{L}^{pseudo}\left(\left\{x_{n}^{m}\right\}_{m=1}^{M},y_{n}\right).\end{aligned}
\end{equation}
For the first part $\mathcal{L}^{\text{fused}}\left(\left\{x_{n}^{m}\right\}_{m=1}^{M}, y_{n}\right)$ is:
\begin{equation}
    \begin{aligned}
    &\quad \mathbb{E}_{\boldsymbol{\mu} \sim \text{Dir}(\boldsymbol{\mu}|\boldsymbol{\alpha})}[\log p(y|\boldsymbol{\mu})]\\
    &\quad - \lambda_{t} D_{\text{PHD}}\left[\text{Dir}(\boldsymbol{\mu}|\widetilde{\boldsymbol{\alpha}}) \| \text{Dir}(\boldsymbol{\mu}|[1,\cdots,1])\right].
\end{aligned}
\end{equation}
For the second part $L^{\text{pseudo}}\left(\left\{x_{n}^{m}\right\}_{m=1}^{M}, y_{n}\right)$:
\begin{equation}
    \begin{aligned}
    &\mathbb{E}_{\boldsymbol{\mu}^{\boldsymbol{p}} \sim \text{Dir}(\boldsymbol{\mu}^{p}|\boldsymbol{\alpha}^{p})}\left[\log p(y \mid \boldsymbol{\mu}^{p})\right] \\
    &\quad - \lambda_{t} D_{\text{PHD}}\left[\text{Dir}(\boldsymbol{\mu}^{p} \mid \widetilde{\boldsymbol{\alpha}}^{p}) \,\|\, \text{Dir}(\boldsymbol{\mu}^{p} \mid [1,\cdots,1])\right].
\end{aligned}
\end{equation}
For the third part:
\begin{align}
    \mathcal{L}^{m}\left(x^{m}, y\right)
    &= \mathbb{E}_{q_{\theta}(\boldsymbol{\mu}^{m}|x^{m})}[\log p(y|\boldsymbol{\mu}^{m})] \nonumber \\
    &- \lambda_{t} \text{PHD}\left[D(\boldsymbol{\mu}^{m}\mid \boldsymbol{\alpha}^{m}) \| D(\boldsymbol{\mu}^{m}\mid [1,\cdots,1])\right].
\end{align}

The primary component in the objective function corresponds to the variational objective function for $M$ integrated modalities. Essentially, this variational objective function involves integrating the traditional cross-entropy loss over a simplex defined by the Dirichlet function. The secondary component serves as a prior constraint to ensure the creation of a more plausible Dirichlet distribution. In essence, the primary variational objective function assesses the KPHD-Net's performance by comparing its predictions to the true labels while imposing constraints on the generation of a more sensible Dirichlet distribution.

The second component within the objective function represents the variational objective function for $M$ integrated pseudo-modalities. The third component within the objective function is focused on deriving the Dirichlet distribution for each individual modality. For a specific modality denoted as “m”, its loss function can be formulated as previously described. The algorithm's workflow is visually outlined in Algorithm \ref{alg:spl}, and the overview of the uncertainty quantification for incomplete multi-view data using divergence measures is shown in Fig. \ref{fig:01}.

\begin{figure}[t]
    \centering
    \includegraphics[scale=0.15]{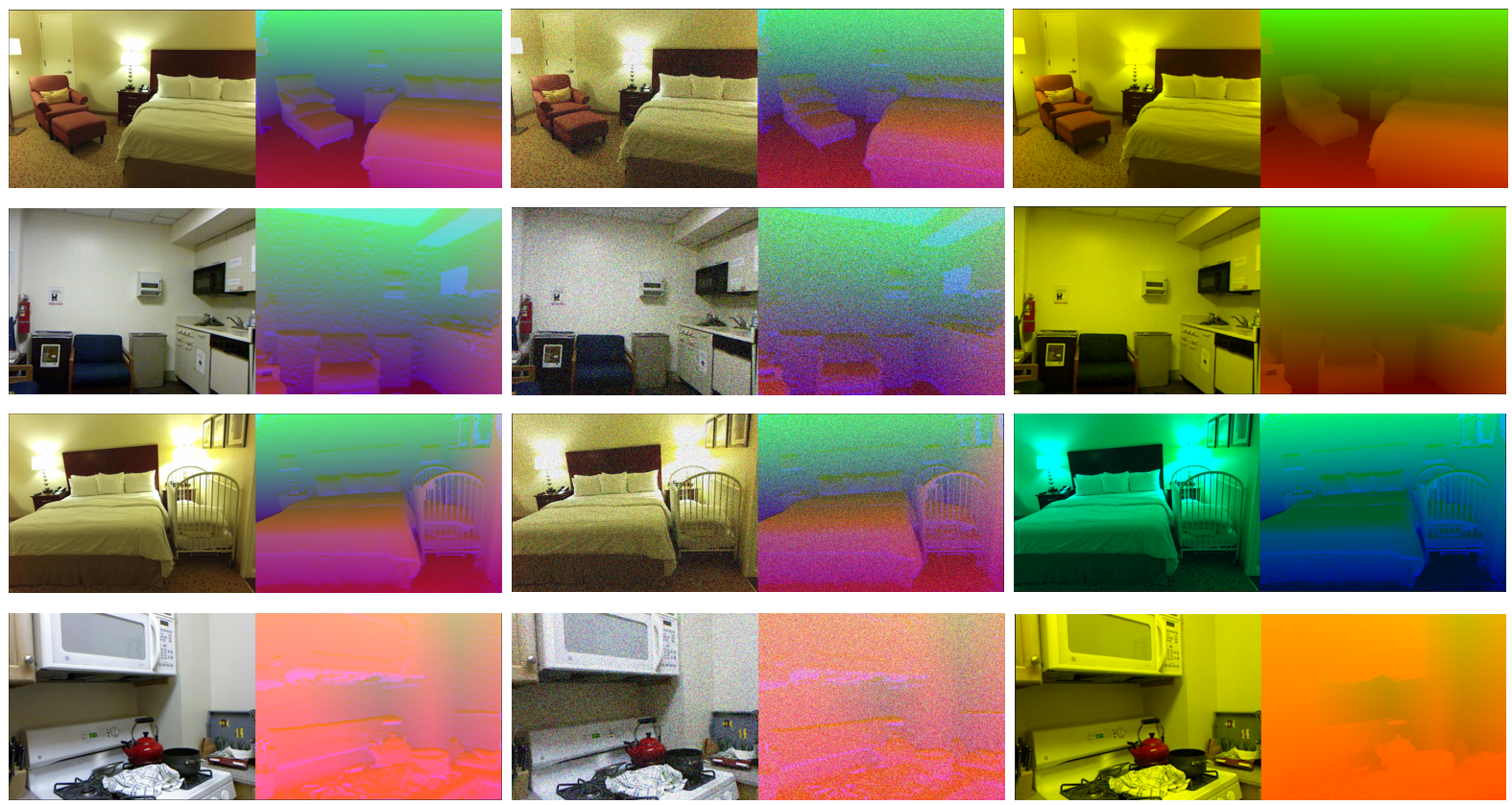}  
    \caption{Comparative Visualization of SUNRGBD and NYUD Datasets: Original images, images corrupted with Gaussian noise ($\sigma^2=0.03$), and images with channel removal at a specified missing rate are displayed. This figure demonstrates the robustness of our method, which effectively classifies images even with channel removal at the designated missing rate.}  
    \label{fig:02}
\end{figure}

\section{Experiment}
In this section, experiments are conducted across various scenarios to provide a comprehensive evaluation of the proposed algorithm. Specifically, we utilize our algorithm to tackle a range of multi-view classification tasks, including RGB-D scene recognition, by employing four real-world multi-view datasets. Additionally, we extend our evaluation beyond classification by incorporating three additional datasets to test the clustering capabilities of KPHD-Net. This comprehensive assessment demonstrates the versatility and effectiveness of our approach in both classification and clustering tasks.
\subsection{Datasets}
\subsubsection{Classification Datasets}
To evaluate the performance of KPHD-Net on multi-view classification tasks, the following datasets are employed: textbf{1. SUNRGBD \cite{p40}}: In the SUN RGB-D dataset, 19 primary scene categories are selected, with each category containing at least 80 images based on prior research. A total of 4,845 samples are used exclusively for training, and an additional 4,659 samples are used for both training and testing. Separate training and testing are then conducted with 20,210 and 3,000 samples, respectively. \textbf{2. NYUDV2 \cite{p41}}: The NYUD Depth V2 (NYUD2) dataset is a compact RGB-D dataset consisting of 1,449 RGB-depth image pairs across 27 categories. These categories are reorganized into 10 classes, including 9 common scenes and one “other” category. A total of 795 samples are used for training and 654 for testing. \textbf{3. ADE20K \cite{p42}}: ADE20K, developed by Stanford University, is designed for semantic segmentation in computer vision. It contains over 20,000 high-resolution images spanning more than 150 object and event categories. The images are reorganized into 10 groups, with 795 samples used for training and 654 for testing. \textbf{4. ScanNet \cite{p43}}: ScanNet consists of 1,513 indoor scenes, each captured with point cloud data from multiple angles. The dataset includes 21 object categories, labeled 0 to 20, with category 0 for unidentified objects. Of the scenes, 1,201 are used for training and 312 for testing.

\subsubsection{Clustering Datasets}

In addition to classification tasks, KPHD-Net's capability in clustering tasks is tested using five multi-view datasets: \textbf{1. MSRC-V1} \cite{p44}: This image dataset contains eight object classes, each with 30 images. Following \cite{p44}, we select seven classes: trees, buildings, airplanes, cows, faces, cars, and bicycles. \textbf{2. ORL} \cite{p73}: This database consists of 10 face images for each of 40 subjects, captured at different times under varying lighting conditions and facial expressions.
\textbf{3. MNIST} \cite{p74}: In this study, we use 4,000 images from the MNIST dataset, a widely recognized benchmark for machine learning and image processing. It contains grayscale images of handwritten digits from 10 categories (digits 0 to 9), with each image being $28\times28$ pixels. The dataset provides a total of 60,000 training images and 10,000 test images. We randomly select 4,000 images to ensure diversity and representativeness for model training and validation. \textbf{4. Caltech101-7} \cite{p45}: The Caltech101-7 dataset, a subset of Caltech101, includes images from seven selected classes, as identified in previous studies \cite{p45}. It is commonly used for training and evaluating object recognition algorithms. \textbf{5. Caltech101-20} \cite{p45}: Caltech101-20 is another subset of Caltech101, consisting of images from 20 selected classes, based on prior research \cite{p45}. It provides a broader range of objects for testing and refining recognition models.

\subsection{Performance Evaluation}
For classification tasks, the performance of KPHD-Net is evaluated using Accuracy \cite{p84}, where TP (true positives), TN (true negatives), FP (false positives), and FN (false negatives) are considered. Accuracy, calculated as: $\text{Accuracy}=\frac{TP+TN}{TP+TN+FP+FN},$ measures overall classification correctness. 

The F1 score is a metric that combines precision and recall to provide a single measure of KPHD-Net's performance in classification tasks. The F1 score \cite{p84} is defined as: $\text{F1} = 2 \times \frac{\text{Precision} \times \text{Recall}}{\text{Precision} + \text{Recall}},$ where precision is the ratio of true positive predictions to the total number of positive predictions made by KPHD-Net, and recall is the ratio of true positive predictions to the total number of actual positives. The F1 score ranges from 0 to 1, with higher values indicating better model performance. For clustering tasks, given a sample \(x_{i}\) for any \(i\in \{1,\ldots,n\}\), the predicted clustering label and the real label are denoted as \(p_{i}\) and \(q_{i}\), respectively. The clustering accuracy (CA) \cite{p60} is defined as: $\mathrm{CA}=\frac{\sum_{i=1}^n\delta(q_i,\mathrm{map}(p_i))}{n},$ where \(\delta(a,b)=1\) if \(a=b\),and \(\delta(a,b)=0\) otherwise. And \(\mathrm{map}(\cdot)\) is the best permutation mapping that matches the predicted clustering labels to the ground truths.

\subsection{Data Preprocessing}
The samples from the specified datasets are merged and preprocessed. In multi-view datasets, images captured from specific angles typically include both RGB and depth images. Prior to training, the image data from these angles are concatenated to optimize the classification process. Additionally, Gaussian noise with variances of 0.01, 0.02, and 0.03, along with missing rates of 0.1, 0.3, and 0.5, is applied to further process the datasets.

\subsection{Model Architecture}
The ResNet50 model \cite{p46}, pretrained on ImageNet \cite{p47}, is used as the base architecture, along with other backbones such as UNet \cite{p78}, DenseNet \cite{p79}, ViT \cite{p80}, and Mamba \cite{p81}, to comprehensively evaluate performance on the clustering task. ResNet50 is a deep residual neural network with 50 layers. Training is conducted on a machine equipped with two Intel(R) Xeon(R) Platinum 8358P @ 2.60GHz CPUs and two 4090 GPUs. Input images are standardized to a size of $256 \times 256$, then randomly cropped to $224 \times 224$. The neural networks are trained using the Adam optimizer \cite{p48}, with weight decay and learning rate decay. For KPHD-Net, pseudo-views are generated by concatenating the outputs of two backbone networks. All experiments are performed using the PyTorch framework \cite{p49}.

\subsection{State-of-the-Art Methods for Performance Comparison in Multi-View Classification and Clustering Tasks} 

 \subsubsection{State-of-the-Art Methods for Performance Comparison in Multi-View Classification Tasks} 
 \textbf{TrecgNet (CVPR, 2019) \cite{p50}} introduced translate-to-recognize networks, a novel approach for RGB-D scene recognition that leverages translation mechanisms to effectively integrate RGB and depth information, thereby enhancing recognition accuracy.  \textbf{G-L-SOOR (IEEE TIP, 2019)\cite{p51}} explored the use of spatial object-to-object relations in image representations to improve RGB-D scene recognition, significantly enhancing scene understanding in complex environments.
 \textbf{CBCL-RGBD (BMVC, 2020)\cite{p52}} presented a centroid-based concept learning approach for RGB-D indoor scene classification, aiming to improve accuracy and robustness by leveraging depth information and centroids to represent key concepts. \textbf{ CMPT (IJCV, 2021)\cite{p53}} introduced a cross-modal pyramid translation approach for RGB-D scene recognition, which enhances the integration of RGB and depth information to improve scene recognition accuracy.  \textbf{CNN-randRNN (CVIU, 2022)\cite{p54}} explored the integration of convolutional neural networks (CNNs) with random recurrent neural networks (RNNs) to develop a multi-level analysis framework for RGB-D object and scene recognition, aiming to boost accuracy and performance.  \textbf{ETMC (TPAMI, 2022)\cite{p18}} introduced a trusted multi-view classification method that employs dynamic evidential fusion to enhance decision-making reliability by effectively integrating information from multiple views.  \textbf{FGMNet {(DSP, 2024)\cite{p85}} proposed a feature grouping network with attention and contextual modules to effectively fuse RGB and depth features for improved indoor scene segmentation.}  \textbf{FCDENet (IOTJ, 2025)\cite{p86}} proposed a feature contrast difference and enhanced network for RGB-D indoor scene classification, utilizing contrast modules, information clustering, and wavelet transforms to improve accuracy and robustness across diverse IoT scenarios.

\subsubsection{State-of-the-Art Methods for Performance Comparison in Multi-View Clustering Tasks} 
 \begin{table}[t]
\centering
\caption{Classification Accuracy on the ADE20K Dataset at Various Missing Rates ($\eta = 0.1, 0.3, 0.5$): This table presents the classification performance on the ADE20K dataset under different levels of missing data, denoted by $\eta$. The results demonstrate the model's robustness and adaptability as the missing rate increases, providing insight into its capability to handle incomplete datasets}
\setlength{\tabcolsep}{2pt} 
\renewcommand\arraystretch{1.0} 
\scriptsize
\begin{tabular}{ccccccc}
\toprule
\multicolumn{7}{c}{$\alpha=2.0 \quad \gamma=0.5$} \\ \midrule
$\eta$ & Depth ACC & RGB ACC& Fusion ACC& Depth F1 & RGB F1& Fusion F1 \\ \midrule
0.1 & 94.10 & 93.75 & 96.18 & 82.47 & 80.46 & 86.69 \\ 
0.3 & 93.71 & 92.36 & 95.76 & 83.64 & 82.44 & 88.43 \\ 
0.5 & 93.54 & 92.29 & 95.69 & 85.62 & 82.84 & 89.76 \\ 
\parbox[t]{1.5cm}{ETMC \cite{p18}} & 85.60 & 85.54 & 89.78 & - & - & - \\ \midrule
\multicolumn{7}{c}{$\alpha=1.1 \quad \gamma=1.8$} \\ \midrule
$\eta$ & Depth ACC & RGB ACC& Fusion ACC& Depth F1 & RGB F1& Fusion F1 \\ \midrule
0.1 & 93.89 & 92.64 & 95.83 & 79.28 & 83.41 & 87.77 \\ 
0.3 & 93.85 & 92.78 & 95.76 & 80.01 & 83.40 & 88.53 \\ 
0.5 & 93.54 & 91.56 & 95.73 & 78.75 & 82.55 & 88.42 \\ 
\parbox[t]{1.5cm}{ETMC \cite{p18}} & 85.60 & 85.54 & 89.78 & - & - & - \\ \midrule
\multicolumn{7}{c}{$\alpha=1.5 \quad \gamma=0.8$} \\ \midrule
$\eta$ & Depth ACC & RGB ACC& Fusion ACC& Depth F1 & RGB F1& Fusion F1 \\ \midrule
0.1 & 93.99 & 92.29 & 96.14 & 82.29 & 83.51 & 89.15 \\ 
0.3 & 93.33 & 92.98 & 95.59 & 82.37 & 83.65 & 87.64 \\ 
0.5 & 93.30 & 91.46 & 95.76 & 83.59 & 81.82 & 89.17 \\ 
\parbox[t]{1.5cm}{ETMC \cite{p18}} & 85.60 & 85.54 & 89.78 & - & - & - \\ \midrule
\multicolumn{7}{c}{$\alpha=1.6 \quad \gamma=0.5$} \\ \midrule
$\eta$ & Depth ACC & RGB ACC& Fusion ACC& Depth F1 & RGB F1& Fusion F1 \\ \midrule
0.1 & 93.82 & 92.36 & 95.69 & 82.48 & 81.78 & 88.08 \\ 
0.3 & 93.57 & 92.57 & 95.87 & 83.61 & 82.28 & 89.05 \\ 
0.5 & 93.85 & 92.50 & 95.69 & 84.69 & 83.36 & 88.79 \\ 
\parbox[t]{1.5cm}{ETMC \cite{p18}} & 85.60 & 85.54 & 89.78 & - & - & - \\ 
\bottomrule
\end{tabular}
\label{tab02}
\end{table}
Several state-of-the-art methods for multi-view clustering are commonly used for performance comparison:
 \textbf{k-Means (TPAMI, 2002) \cite{p61}} introduced an efficient k-means clustering algorithm, emphasizing its analysis and implementation to enhance performance in data clustering tasks.  \textbf{SWMC (IJCAI, 2017)\cite{p55}} proposed a self-weighted multi-view clustering approach that utilizes multiple graphs to improve clustering performance across different data views. \textbf{ MLAN (AAAI, 2017)\cite{p56}} presented methods for multi-view clustering and semi-supervised classification, using adaptive neighbors to enhance performance.  \textbf{MCGC (IEEE TIP, 2018)\cite{p58}} introduced Multiview Consensus Graph Clustering, a novel approach that leverages consensus across multiple views to enhance clustering accuracy and robustness in image processing tasks.  \textbf{BMVC (TPAMI, 2018)\cite{p59}} explored binary multi-view clustering, focusing on techniques for effectively clustering data from multiple perspectives using binary representations.  \textbf{MSC-IAS (PR, 2019)\cite{p57}} developed a multi-view subspace clustering method that incorporates intactness-aware similarity to improve clustering outcomes.  \textbf{DSRL (TPAMI, 2022)\cite{p60}} proposed a method for learning deep sparse regularizers, aimed at enhancing multi-view clustering and semi-supervised classification tasks.  \textbf{MCBVA(IJMLC, 2023) \cite{p76}} proposed a clustering method named Multi-view Clustering Based on View-Attention Driven. Based on autoencoder, the method learns feature representations from different views data using contrast learning and attention mechanisms.  \textbf{DSCMC (ArXiv, 2024)\cite{p75}} proposed a novel multi-view clustering model. The main objective of approach is to enhance the clustering performance by leveraging cotraining in two distinct spaces.  \textbf{DIMvLN (AAAI, 2024)\cite{p77}} proposed a novel Deep Incomplete Multi-view Learning Network (DIMvLN). Firstly designs the deep graph networks to effectively recover missing data with assigning pseudo-labels of large amounts of unlabeled instances. 

\begin{table}[t]
	\setlength{\belowdisplayskip}{0pt}
	\setlength{\abovedisplayskip}{0pt}
	\setlength{\abovecaptionskip}{0pt}
	\centering
	\scriptsize
	\caption{Quantitative Evaluation of Intra-Class Accuracy Results on NYUD Depth V2, ADE20K, ScanNet, and SUN RGB-D Datasets. This table presents a comprehensive quantitative evaluation of intra-class accuracy for several widely-used datasets, including NYUD Depth V2, ADE20K, ScanNet, and SUN RGB-D. The results highlight the performance of the model in accurately distinguishing between various classes within each dataset, providing a comparative analysis across different domains of depth, scene understanding}
	\setlength{\tabcolsep}{5pt}
	\renewcommand\arraystretch{0.8}
	\begin{tabular}{p{1.0cm}p{3cm}p{1.1cm}p{1.0cm}p{1.1cm}}  
		\toprule [1.0pt]
		{Datasets}&	Models &Depth (\%) &RGB (\%) &Fusion (\%)\\
		\midrule[0.5pt]
        \multirow{7}{*}{NYUD2} 
        & Ours ($\alpha = 1.1$,$\gamma = 2.0$) & 63.86 & 66.71 & 72.14 \\
        & Ours ($\alpha = 1.5$,$\gamma = 5.0$) & 64.16 & \textbf{67.32} & 72.60 \\
        & Ours ($\alpha = 1.6$,$\gamma = 5.0$) & 64.31 & 66.72 & 72.45 \\
        & Ours ($\alpha = 2.0$,$\gamma = 8.0$) & 64.41 & 65.67 & 73.34 \\
        & Ours ($\alpha = 2.0$,$\gamma = 5.0$) & \textbf{66.27} & 66.27 & \textbf{73.64} \\
        & ETMC \cite{p18} & 65.51 & 64.91 & 72.43 \\
        \midrule[0.5pt]
        \multirow{6}{*}{ADE20K} 
        & Ours ($\alpha = 1.1$,$\gamma = 1.8$) & 93.75 & 93.33 & 96.11 \\
        & Ours ($\alpha = 1.5$,$\gamma = 0.8$) & 94.44 & 92.46 & 96.14 \\
        & Ours ($\alpha = 1.6$,$\gamma = 1.0$) & 94.31 & 92.71 & 96.18 \\
        & Ours ($\alpha = 2.0$,$\gamma = 1.5$) & 94.23 & 92.01 & 96.04 \\
        & Ours ($\alpha = 2.0$,$\gamma = 0.5$) & \textbf{96.62} & \textbf{93.78} & \textbf{96.21} \\
        & ETMC \cite{p18} & 85.60 & 85.54 & 89.78 \\
        \midrule[0.5pt]
        \multirow{6}{*}{ScanNet} 
        & Ours ($\alpha = 1.1$,$\gamma = 0.5$) & 72.54 & 89.68 & 89.71 \\
        & Ours ($\alpha = 1.5$,$\gamma = 0.5$) & 76.11 & 90.83 & 90.86 \\
        & Ours ($\alpha = 1.6$,$\gamma = 0.5$) & 76.76 & 91.75 & \textbf{92.14} \\
        & Ours ($\alpha = 2.0$,$\gamma = 8.0$) & 77.55 & 91.60 & 92.08 \\
        & Ours ($\alpha = 2.0$,$\gamma = 0.5$) & \textbf{78.49} & \textbf{92.14} & 92.09 \\
        & ETMC \cite{p18} & 75.89 & 90.71 & 91.63 \\
        \midrule[0.5pt]
        \multirow{6}{*}{SUN RGB-D} 
        & Ours ($\alpha = 1.1$,$\gamma = 0.5$) & 51.23 & 54.97 & 59.19 \\
        & Ours ($\alpha = 1.5$,$\gamma = 0.5$) & 53.53 & 55.53 & 60.18 \\
        & Ours ($\alpha = 1.6$,$\gamma = 0.5$) & 52.71 & 56.03 & 60.40 \\
        & Ours ($\alpha = 1.8$,$\gamma = 0.8$) & 53.51 & \textbf{56.64} & 60.95 \\
        & Ours ($\alpha = 2.0$,$\gamma = 1.0$) & \textbf{53.94} & 56.43 & \textbf{61.60} \\
        & ETMC \cite{p18} & 52.48 & \textbf{56.64} & 60.80 \\
        \bottomrule[1.0pt]
	\end{tabular}
	\label{tab03}
\end{table}
\begin{table}
\centering
\scriptsize
\caption{Quantitative Evaluation of Inter-Class Accuracy for Experimental Results on NYUD Depth V2 and SUN RGB-D Datasets. This table presents a comprehensive quantitative assessment of inter-class accuracy achieved from experimental results on the NYUD Depth V2 and SUN RGB-D datasets. It highlights the performance distinctions between different classes, providing critical insights into the model's classification precision across diverse categories}
\setlength{\tabcolsep}{5pt}
\renewcommand\arraystretch{1.0}
\begin{tabular}{p{1.2cm}p{2.2cm}p{1.1cm}p{1.1cm}p{1.1cm}}
\toprule[1.0pt]
Datasets & Models & RGB (\%) & Depth (\%) & Fusion (\%) \\
\midrule[0.5pt]
\multirow{8}{*}{NYUD2} 
& TrecgNet \cite{p50} & 64.80 & 57.70 & 69.20 \\
& G-L-SOOR \cite{p51} & 64.20 & 62.30 & 67.40 \\
& CBCL-RGBD \cite{p52} & 66.40 & 49.50 & 70.90 \\
& CMPT \cite{p53} & 66.10 & 64.10 & 71.80 \\
& CNN-randRNN \cite{p54} & \textbf{69.10} & 64.00 & 73.00 \\
& ETMC \cite{p18} & 64.91 & 65.51 & 72.43 \\
& FGMNet \cite{p85} & - & - & 69.0 \\
& FCDENet \cite{p86} & - & - & 72.0 \\
& Ours & 66.27 & \textbf{66.27} & \textbf{73.64} \\
\midrule[0.5pt]
\multirow{7}{*}{SUN RGB-D} 
& TrecgNet \cite{p50} & 50.60 & 47.90 & 56.70 \\
& G-L-SOOR \cite{p51} & 50.40 & 44.10 & 55.50 \\
& CBCL-RGBD \cite{p52} & 48.80 & 37.30 & 50.90 \\
& CMPT \cite{p53} & 54.20 & 49.40 & 61.10 \\
& CNN-randRNN \cite{p54} & \textbf{58.50} & 50.40 & \textbf{62.50} \\
& ETMC \cite{p18} & 56.64 & 52.48 & 60.80 \\
& Ours & 53.94 & \textbf{56.43} & 61.60 \\
\bottomrule[1.0pt]
\end{tabular}
\label{tab04}
\end{table}

\subsection{Results of the Classification}
\subsubsection{Experimental Analysis}
Table \ref{tab02} shows the KPHD-Net's performance under varying missing rates ($\eta$), evaluate by accuracy and F1 score. For $\alpha=2.0$ and $\gamma=0.5$, KPHD-Net maintains high accuracy and F1 scores, with fusion accuracy of 96.18\% and F1 score of 86.69 at a missing rate of 0.1. As the missing rate increases to 0.5, accuracy slightly drops to 95.69\%, but the F1 score improves to 89.76\%. With $\alpha=1.1$ and $\gamma=1.8$, KPHD-Net achieves a fusion accuracy of 95.83\% and an F1 score of 87.77\% at a missing rate of 0.1, consistently outperforming the ETMC baseline. For $\alpha=1.5$ and $\gamma=0.8$, peak fusion accuracy reaches 96.14\% with an F1 score of 89.15\% at a missing rate of 0.1. Performance remains stable across different missing rates. The KPHD-Net demonstrates strong robustness to missing data, effectively managing uncertainty and surpassing the ETMC.

Intra-class experiments are conducted to test KPHD-Net with various hyperparameters on real-world scene datasets. Four distinct datasets, namely NYUD Depth V2, ADE20K, ScanNet, and SUN RGB-D, are used to evaluate classification performance in complex scenarios. The experimental results are provided in Table \ref{tab03}.

In these experiments, KPHD-Net is assessed across four multi-class datasets, and its performance is compared to that of ETMC. The Hölder index in KPHD-Net is adaptively adjusted for different datasets to optimize performance. During testing, the classification accuracy of individual modalities (Depth and RGB) and fused modalities is evaluated separately.

For the NYUD Depth V2 dataset, KPHD-Net is shown to achieve superior performance, with a fusion modality accuracy of 73.64\%, reflecting a 1.21\% improvement over ETMC. Notably, the RGB modality accuracy improves to 67.32\% with $\alpha = 1.5$ and $\gamma = 5.0$, surpassing ETMC.

In the ADE20K dataset, KPHD-Net achieves a fusion modality accuracy of 96.21\%, demonstrating a significant improvement of 6.43\% compared to ETMC. The highest individual modality accuracy, 96.62\%, is observed for the depth modality with $\alpha = 2.0$ and $\gamma = 0.5$.

For the ScanNet dataset, KPHD-Net reaches a fusion modality accuracy of 92.14\%, reflecting a 0.51\% improvement over ETMC. The highest individual modality accuracy, 92.14\%, is recorded for the RGB modality with $\alpha = 2.0$ and $\gamma = 0.5$.

In the SUN RGB-D dataset, KPHD-Net achieves a fusion modality accuracy of 61.60\%, 0.80\% higher than ETMC. The depth modality accuracy is 53.94\% with $\alpha = 2.0$ and $\gamma = 1.0$.

The experimental results indicate that KPHD-Net consistently outperforms ETMC across different datasets and modalities. The introduction of PHD and Kalman filtering allows KPHD-Net to better adapt to the diverse characteristics of multi-class data, providing more accurate and reliable uncertainty estimates. This leads to enhanced performance through more effective integration of modality information.

Inter-class experiments compare KPHD-Net with pre-existing algorithms on the NYUD Depth V2 and SUN RGB-D datasets. KPHD-Net shows superior performance, achieving the highest fusion accuracy of 73.64\% and 61.60\%, respectively. The results highlight the positive impact of uncertainty analysis on multi-view classification accuracy, especially in fused modalities. Enhanced uncertainty analysis and the refined objective function contribute to KPHD-Net’s improved performance.

The results in Table \ref{tab04} show the inter-class accuracy for various models on the NYUD Depth V2 and SUN RGB-D datasets. On the NYUD Depth V2 dataset, our model outperforms many existing models in the depth modality, achieving the highest score of 66.27\%, surpassing CNN-randRNN (64.00\%). However, it slightly falls behind CNN-randRNN in the RGB modality (66.27\% vs. 69.10\%) and in the fusion modality (73.64\% vs. 73.00\%). This may be attributed to the fact that while our model performs well in depth feature extraction, it does not capture RGB information as effectively as CNN-randRNN, which is specifically designed to handle RGB data more robustly. The fusion performance gap suggests that our method may benefit from further refinement in combining these modalities. On the SUN RGB-D dataset, our model achieves 53.94\% in the RGB modality, which is lower than CNN-randRNN (58.50\%), but it surpasses it in the depth modality with 56.43\% compared to 50.40\%. In fusion, our model scores 61.60\%, slightly behind CNN-randRNN (62.50\%). This performance discrepancy in the fusion condition may result from a less optimal interaction between RGB and depth features in our model, which could be improved by exploring more advanced fusion strategies. Overall, while our model demonstrates strong performance, particularly with depth data, further refinement in RGB feature extraction and fusion strategies could yield even better results. This analysis highlights the need for a deeper exploration of how our model processes RGB features and combines them with depth data to enhance its performance, especially in comparison to CNN-randRNN.

Table \ref{tab06} presents KPHD-Net's performance on the ADE20K dataset under various noise conditions. With $\sigma^2=0.01$, KPHD-Net achieves peak fusion accuracy of 94.65\% for $\alpha=2.0$ and $\gamma=0.5$, closely approaching the no-noise accuracy of 96.21\%. The RGB modality accuracy remains stable at 80.45\%. For $\sigma^2=0.03$, accuracy decreases but still exceeds the ETMC baseline of 89.78\%, with fusion accuracy at 92.53\% and RGB accuracy at 75.17\%. Overall, KPHD-Net demonstrates effective noise handling, maintaining high accuracy through robust uncertainty analysis.
\begin{table}[t]
    \centering
    \scriptsize 
    \caption{Clustering Performance on the Caltech101-20 Dataset Across Different Hyperparameter Configurations. This table presents the clustering results obtained on the Caltech101-20 dataset, evaluated under various hyperparameter settings, highlighting their impact on model performance}
    
    \setlength{\tabcolsep}{4pt} 
    \renewcommand\arraystretch{1.1} 
    \begin{tabular}{cccc|cccc}
        \toprule[1.0pt]
        \multicolumn{4}{c}{$\alpha=1.7$} & \multicolumn{4}{c}{$\alpha=1.1$} \\
        \midrule[0.5pt]
        $\gamma$ & Depth & RGB & Fusion & $\gamma$ & Depth & RGB & Fusion \\ 
        \midrule[0.5pt]
        0.5 & 58.86 & 64.52 & 90.13 & 0.5 & 65.62 & 67.26 & 89.47 \\ 
        1.3 & \textbf{67.79} & 68.11 & 90.85 & 1.3 & \textbf{72.06} & \textbf{68.08} & 89.85 \\ 
        2 & 67.17 & 65.94 & \textbf{92.48} & 2 & 69.29 & 66.76 & \textbf{90.13} \\ 
        5 & 66.74 & \textbf{70.71} & 89.70 & 5 & 61.80 & 63.07 & 86.72 \\ 
        10 & 63.52 & 66.42 & 91.14 & 10 & 54.44 & 56.57 & 79.20 \\ 
        \parbox[t]{1.3cm}{DSRL \cite{p60}} & - & - & 72.9 & \parbox[t]{1.3cm}{DSRL \cite{p60}} & - & - & 72.9 \\ 
        \bottomrule[1.0pt]
    
        \multicolumn{4}{c}{$\alpha=1.3$} & \multicolumn{4}{c}{$\alpha=2.0$} \\
        \midrule[0.5pt]
        $\gamma$ & Depth & RGB & Fusion & $\gamma$ & Depth & RGB & Fusion \\ 
        \midrule[0.5pt]
        0.5 & 65.77 & 67.34 & 91.84 & 0.5 & 67.94 & 67.43 & 92.18 \\ 
        1.3 & \textbf{72.42} & 69.47 & \textbf{92.44} & 1.3 & 64.30 & \textbf{69.06} & 92.44 \\ 
        2 & 72.23 & 66.38 & 90.72 & 2 & 70.82 & 68.15 & \textbf{92.59} \\ 
        5 & 70.07 & \textbf{69.65} & 90.36 & 5 & 65.85 & 68.90 & 91.22 \\ 
        10 & 62.97 & 53.61 & 87.68 & 10 & \textbf{73.56} & 63.22 & 91.14 \\ 
        \parbox[t]{1.3cm}{DSRL \cite{p60}} & - & - & 72.9 & \parbox[t]{1.3cm}{DSRL \cite{p60}} & - & - & 72.9 \\
        \bottomrule[1.0pt]
    \end{tabular}
    \label{tab05}
\end{table}

\begin{figure*}
    \setlength{\belowcaptionskip}{0pt} 
    \setlength{\abovecaptionskip}{0pt} 
    \centering
    \includegraphics[width=0.9\linewidth]{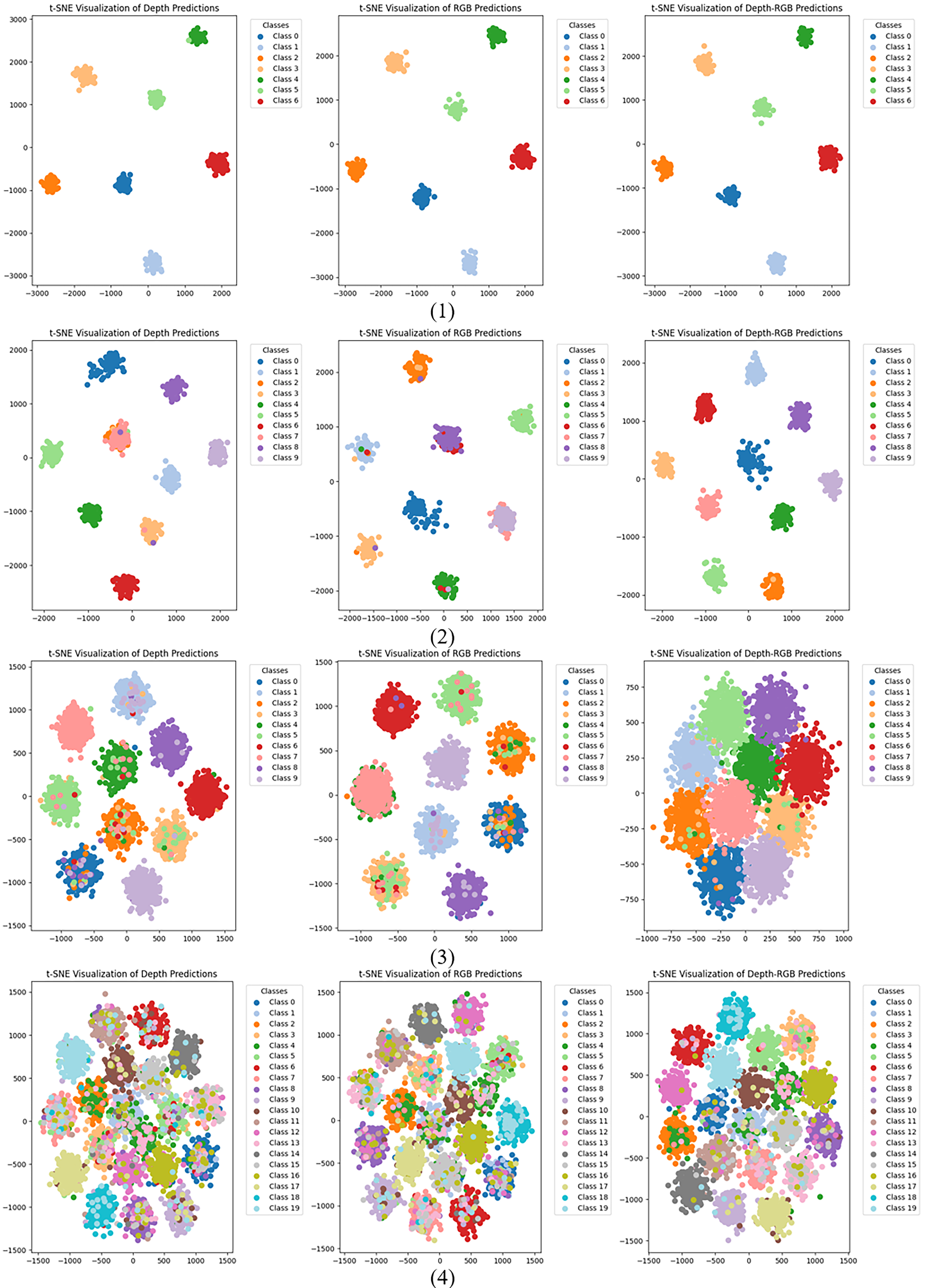}
    \caption{t-SNE visualizations of multi-view clustering results across diverse datasets: (1) MSRC-v1, (2) MNIST, (3) Caltech101-7, and (4) Caltech101-20. This table highlights the clustering performance of our model across these datasets, offering a comparative analysis of the outcomes.}
    \label{fig:10000}
\end{figure*}
\begin{table*}[htbp]
    \centering
    \scriptsize 
    \caption{Summary of Experimental Results on the ADE20K Dataset with Varying Levels of Gaussian Noise. This table presents a detailed comparison of performance metrics when applying different levels of Gaussian noise to the ADE20K dataset. The results are categorized by varying noise parameters, where the depth, RGB, and fusion model outputs are measured under different values of $\alpha$ and $\gamma$. The performance is reported both with and without noise, along with a reference comparison to the ETMC method}
    \setlength{\tabcolsep}{10pt} 
    \renewcommand\arraystretch{0.8} 
    \begin{tabular}{c c c c c c c c c c c c}
        \toprule[1.0pt]
        \multicolumn{4}{c}{$\alpha=1.5\,\sigma^2=0.01$} & \multicolumn{4}{c}{$\alpha=1.8\,\sigma^2=0.01$} & \multicolumn{4}{c}{$\alpha=2.0\,\sigma^2=0.01$}\\
        \cmidrule(lr){1-4} \cmidrule(lr){5-8} \cmidrule(lr){9-12}
        $\gamma$ & Depth & RGB & Fusion & $\gamma$ & Depth & RGB & Fusion  & $\gamma$ & Depth & RGB & Fusion  \\
        \midrule[0.5pt]
        0.5 & 86.31 & 82.01 & 94.17 & 0.5 & 84.68 & 82.98 & 93.96 & 0.5 & 88.05 & 80.45 & 94.65\\
        no noise & 93.85 & 92.85 & 96.04 & no noise & 94.06 & 92.50 & 96.08 & no noise & 94.62 & 93.78 & 96.21\\
        1.3 & 85.52 & 85.62 & 94.58 & 1.3 & 86.66 & 83.78 & 94.72 & 1.3 & 84.68 & 84.54 & 94.27\\
        no noise & 93.02 & 92.25 & 95.62 & no noise & 93.40 & 93.09 & 95.80 & no noise & 93.89 & 91.77 & 95.94\\
        2.0 & 87.01 & 82.95 & 94.58 & 2.0 & 79.19 & 81.63 & 94.03 & 2.0 & 82.46 & 83.92 & 94.37\\
        no noise & 93.30 & 93.40 & 95.97 & no noise & 93.99 & 92.60 & 96.01 & no noise & 93.68 & 92.60 & 95.97\\
        ETMC \cite{p18} & 85.60 & 85.54 & 89.78 & ETMC \cite{p18} & 85.60 & 85.54 & 89.78 & ETMC \cite{p18} & 85.60 & 85.54 & 89.78 \\
        \midrule[0.5pt]
        \multicolumn{4}{c}{$\alpha=1.5\,\sigma^2=0.03$} & \multicolumn{4}{c}{$\alpha=1.8\,\sigma^2=0.03$} & \multicolumn{4}{c}{$\alpha=2.0\,\sigma^2=0.03$}\\
        \cmidrule(lr){1-4} \cmidrule(lr){5-8} \cmidrule(lr){9-12}
        $\gamma$ & Depth & RGB & Fusion & $\gamma$ & Depth & RGB & Fusion & $\gamma$ & Depth & RGB & Fusion \\
        \midrule[0.5pt]
        0.5 & 87.81 & 65.76 & 92.53 & 0.5 & 86.18 & 74.58 & 92.12 & 0.5 & 86.14 & 75.17 & 91.94\\
        no noise & 93.85 & 92.85 & 96.04 & no noise & 94.06 & 92.50 & 96.08 & no noise & 94.62 & 93.78 & 96.21\\
        1.3 & 83.64 & 74.78 & 92.43 & 1.3 & 87.18 & 67.45 & 92.46 & 1.3 & 87.36 & 68.98 & 92.29\\
        no noise & 93.02 & 92.25 & 95.62 & no noise & 93.40 & 93.09 & 95.80 & no noise & 93.89 & 91.77 & 95.94\\
        2.0 & 78.50 & 78.64 & 93.23 & 2.0 & 86.59 & 75.41 & 92.57 & 2.0 & 86.91 & 76.49 & 92.32\\
        no noise & 93.30 & 93.40 & 95.97 & no noise & 93.99 & 92.60 & 96.01 & no noise & 93.68 & 92.60 & 95.97\\
        ETMC \cite{p18} & 85.60 & 85.54 & 89.78 & ETMC \cite{p18} & 85.60 & 85.54 & 89.78 & ETMC \cite{p18} & 85.60 & 85.54 & 89.78 \\
        \bottomrule[1.0pt]
    \end{tabular}
    \label{tab06}
\end{table*}

\begin{table*}[htbp]
	\setlength{\belowdisplayskip}{0pt}
	\setlength{\abovedisplayskip}{0pt}
	\setlength{\abovecaptionskip}{0pt}
	\centering
	\scriptsize
	\caption{Clustering accuracy comparison of various multi-view clustering methods. This table provides a comprehensive comparison of clustering accuracy across various multi-view clustering algorithms, showcasing the superior performance of our model in handling multiple perspectives}
	\setlength{\tabcolsep}{20pt}
	\begin{tabular}{l|*{5}{c}}
        \toprule[1.0pt]
		\backslashbox{Models}{Datasets }&\makebox[1.0em]{MSRC-v1}&\makebox[3em]{MNIST}&\makebox[3em]{ORL}&\makebox[4em]{Caltech101-7}&\makebox[4em]{Caltech101-20}\\
        \midrule[0.5pt]
        K-Means \cite{p61} & 46.3 (1.7) &  73.9 (7.2) &  59.0 (2.4)  &  49.6 (5.8)  &  31.3 (2.5)\\
        MLAN \cite{p56}& 68.1 (0.0) &  77.1 (0.5) &  77.8 (0.0)  &  78.0 (0.0)  &  52.6 (0.8) \\
        SWMC \cite{p55}&78.6 (0.0) &  77.9 (0.0) &  74.8 (0.0)  &  66.5 (0.0)  &  54.1 (0.0) \\
        MSC-IAS \cite{p57}& 47.5 (2.0) &  74.8 (0.3) &  73.3 (2.2)  &  71.3 (4.3)  &  41.9 (2.7) \\
        MCGC \cite{p58}& 75.2 (0.0) & 88.7 (0.0) &  81.0 (0.0)   &  64.3 (0.0)  &  54.6 (0.0) \\
        BMVC \cite{p59}& 63.8 (0.0) &  62.6 (2.2) &  56.7 (0.0)  &  57.9 (0.0)  &  47.4 (0.7) \\
        DSRL \cite{p60}& 83.4 (4.3) &  85.6 (0.3) &  83.6 (1.2)  &  83.8 (1.7)  &  72.9 (1.1) \\
        DSCMC \cite{p75}& - & - &  -  &  86.16 (0.0)  &  76.19 (0.0) \\
        MCBVA \cite{p76}& - & 90.31 (0.0) &  -   &  -  &  74.88 (0.0) \\
        DIMvLN \cite{p77}& - & - &  -   &  -  &  91.28 (1.1) \\
        Ours & \textbf{100 (2.1)} &  \textbf{99.25 (2.6)} &  \textbf{100 (1.9)}  & \textbf{98.37 (3.9)}  &  \textbf{92.59 (2.6)} \\
		\bottomrule[1.0pt]
	\end{tabular}
	\label{tab07}
\end{table*}

\subsection{Results of the Clustering}
\subsubsection{Experimental Analysis}
The experimental results presented in Table \ref{tab06} demonstrate the impact of the Hölder index ($\alpha$) and the parameter $\gamma$ on the performance of KPHD-Net across Depth, RGB, and Fusion modalities. Moderate values of $\alpha$ (e.g., $a=1.7$ or $a=2.0$) and $\gamma$ (e.g., $\gamma=2$) are shown to yield the highest fusion accuracies, reaching up to 92.59\%. This indicates that the PHD is effectively used to adjust KPHD-Net's sensitivity to varying data distributions, leading to more accurate fusion results. These findings highlight KPHD-Net's ability to efficiently integrate multi-view data, enhance clustering accuracy, and demonstrate robustness across different hyperparameter settings.

The experimental results shown in Table \ref{tab07} further confirm KPHD-Net's superior performance across the MSRC-v1, Caltech101-7, and Caltech101-20 datasets, consistently surpassing state-of-the-art multi-view clustering methods. In Fig. \ref{fig:10000}, a visualization of multi-view clustering results across different datasets is provided. The high-dimensional input data is projected onto a two-dimensional subspace using t-SNE, with corresponding data points marked by different colors based on their predicted labels. On the MSRC-v1 dataset, KPHD-Net achieves an unprecedented clustering accuracy of 100\% in the fusion modality, significantly surpassing both traditional and advanced methods. This exceptional result is mirrored on the Caltech101-7 dataset, where an accuracy of 98.37\% is attained, and on the more challenging Caltech101-20 dataset, with KPHD-Net achieving 92.59\%, outperforming competing methods.

It is noted that the experimental results successfully demonstrate the effectiveness of the multi-view fusion method. While the accuracy for individual views is relatively low, the fusion strategy significantly enhances the results. The remarkable performance of KPHD-Net is attributed to its innovative techniques. The integration of PHD and Kalman filtering allows for more flexible and accurate fusion of diverse data views while improving uncertainty estimation. Additionally, KPHD-Net's dynamic weighting of different views enables it to handle noisy or corrupted data effectively, a common challenge in real-world scenarios.

These results highlight KPHD-Net's effectiveness not only in clustering tasks but also in its versatility across both clustering and classification tasks. Unlike traditional multi-view learning models, which may excel in one domain but struggle in another, KPHD-Net consistently delivers superior performance across both, establishing itself as a robust and comprehensive solution for multi-view learning challenges.

\subsection{Ablation Study}

To comprehensively evaluate the impact of each proposed component, we conduct an ablation study on three benchmark datasets: ADE20K, NYUD2, and SUN RGB-D. Tables~\ref{tab:ablation_combined}, present the performance across three experimental configurations—KL divergence only, Hölder divergence, and Dirichlet-based Hölder divergence—under three input settings (RGB, Depth, and fusion). For ADE20K, e fix the divergence parameter at $\alpha=1.7$, while for NYUD2 and SUN RGB-D,$\alpha=2.0$ is used.

Across all datasets, using only KL divergence yields the lowest accuracy. For example, on ADE20K, KL achieves 85.54\% (RGB), 85.60\% (Depth), and 89.78\% (Fusion); on NYUD2, the scores are 65.51\%, 64.91\%, and 72.43\%, respectively; and on SUN RGB-D, 52.48\%, 56.64\%, and 60.80\%. These results suggest that while KL divergence captures some uncertainty, it provides limited benefit in multimodal fusion.

Replacing KL with Hölder divergence leads to consistent performance gains. On ADE20K, accuracy improves to 94.06\% (RGB), 92.50\% (Depth), and 96.08\% (Fusion). Similar trends are observed on NYUD2 (66.27\% across RGB and Depth; 73.64\% Fusion) and SUN RGB-D (53.94\%, 56.43\%, and 61.60\%). This confirms that Hölder divergence better models distributional differences between modalities, especially in complex scenes.

When introducing a Dirichlet prior into the Hölder framework, performance varies slightly. On ADE20K, fusion accuracy further improves to 96.18\%, although RGB accuracy drops slightly to 89.65\%. On NYUD2 and SUN RGB-D, however, Dirichlet modeling slightly degrades fusion performance compared to plain Hölder (70.33\% vs. 73.64\% on NYUD2, 61.10\% vs. 61.60\% on SUN RGB-D). This suggests that the Dirichlet prior benefits datasets with more structured or heterogeneous views (like ADE20K), but may not generalize as effectively on noisier datasets.

In summary, Hölder divergence clearly outperforms KL divergence across all datasets and modalities. Incorporating a Dirichlet prior brings additional benefits in high-quality or diverse-view settings (e.g., ADE20K), but may introduce marginal trade-offs in lower-quality datasets. These findings validate the design choices behind our uncertainty modeling framework.
\begin{table}
    \centering
    \caption{Quantitative evaluation results of the ablation study on the ADE20K ($\alpha=1.7$), NYUD2 ($\alpha=2.0$), and SUN RGB-D ($\alpha=2.0$) datasets. The table shows the impact of different modules or design choices on model performance across all datasets, confirming the effectiveness and necessity of the proposed components in improving the results.}
    \label{tab:ablation_combined}
    \setlength{\tabcolsep}{1.6pt}
    \begin{tabular}{lcccccc}
        \toprule
        \textbf{Dataset} & \textbf{KL} & \textbf{H\"{o}lder} & \textbf{H\"{o}lder (Dir)} & \textbf{RGB (\%)} & \textbf{Depth (\%)} & \textbf{Fusion (\%)} \\
        \midrule
        ADE20K    & \checkmark & $\times$ & $\times$ & 85.54 & 85.60 & 89.78 \\
        ADE20K    & $\times$   & \checkmark & $\times$ & 94.06 & 92.50 & 96.08 \\
        ADE20K    & $\times$   & $\times$ & \checkmark & 89.65 & 93.64 & 96.18 \\
        NYUD2     & \checkmark & $\times$ & $\times$ & 65.51 & 64.91 & 72.43 \\
        NYUD2     & $\times$   & \checkmark & $\times$ & 66.27 & 66.27 & 73.64 \\
        NYUD2     & $\times$   & $\times$ & \checkmark & 62.35 & 61.60 & 70.33 \\
        SUN RGB-D & \checkmark & $\times$ & $\times$ & 52.48 & 56.64 & 60.80 \\
        SUN RGB-D & $\times$   & \checkmark & $\times$ & 53.94 & 56.43 & 61.60 \\
        SUN RGB-D & $\times$   & $\times$ & \checkmark & 50.40 & 55.80 & 61.10 \\
        \bottomrule
    \end{tabular}
\end{table}

\subsection{Hyperparameter Selection via Grid Search Experiments}

The parameters $\alpha$ and $\gamma$ play a critical role in the Proper Hölder Divergence, as their values directly influence model performance. While prior studies often rely on manual, empirically driven tuning, we adopt an automated grid search approach to systematically identify optimal hyperparameters, thereby improving the algorithm’s adaptability and intelligence.

Table~\ref{tabwangge} reports the results of the grid search experiments for hyperparameter selection ($\alpha$ and $\gamma$) on the ADE20K dataset. We evaluate model accuracy (\%) across three modalities: Depth, RGB, and Fusion.

Overall, the Fusion modality consistently outperforms both Depth and RGB across all parameter settings, underscoring the effectiveness of multi-view feature integration. Notably, when $\alpha=1.7$ and $\gamma=0.5$, Fusion achieves the highest accuracy of \textbf{96.14\%}. Other parameter pairs, such as $(\alpha=2.0, \gamma=1.3)$ and $(\alpha=1.1, \gamma=1.0)$, also yield high Fusion accuracy of \textbf{96.14\%} and \textbf{96.01\%}, respectively.

As $\alpha$ increases, the performance of both Depth and RGB modalities fluctuates, with a slight decline observed for large $\alpha$ values (e.g., $\alpha=2.5$), particularly when paired with higher $\gamma$. This trend suggests that excessive regularization or overly strong divergence constraints (i.e., high $\alpha$ and $\gamma$) can impair the performance of individual modalities.

Interestingly, the Depth modality alone often performs competitively and occasionally surpasses RGB, especially when $\alpha=1.3$ and $\gamma=0.5$ (Depth: 94.48\%, RGB: 93.16\%), highlighting the importance of geometric cues in semantic segmentation tasks.

In summary, the grid search results indicate that optimal performance is achieved with moderate values of $\alpha$ and $\gamma$. The Fusion modality benefits most from this tuning, with the best results typically observed when $\alpha$ ranges from 1.7 to 2.0 and $\gamma$ ranges from 0.5 to 1.5.
\begin{table*}[htbp]
    \centering
    \scriptsize 
    \caption{Grid Search Results for Hyperparameter Selection on the ADE20K Dataset with Varying $\alpha$ ($1.1, 1.3, 1.5, 1.7, 2.0, 2.5$) and $\gamma$ ($0.5, 0.8, 1.0, 1.3, 1.5, 2.0$). This table reports the results of a comprehensive grid search performed on the ADE20K dataset, systematically evaluating model performance across all combinations of the hyperparameters $\alpha$ and $\gamma$. The outcomes facilitate identification of the optimal parameter settings, offering insights into the model’s sensitivity to these hyperparameters and providing empirical guidance for future experiments and practical deployments.}
    \setlength{\tabcolsep}{12pt} 
    \renewcommand\arraystretch{0.8} 
    \begin{tabular}{c c c c c c c c c c c c}
        \toprule[1.0pt]
        \multicolumn{4}{c}{$\alpha=1.1$} & \multicolumn{4}{c}{$\alpha=1.3$} & \multicolumn{4}{c}{$\alpha=1.5$}\\
        \cmidrule(lr){1-4} \cmidrule(lr){5-8} \cmidrule(lr){9-12}
        $\gamma$ & Depth & RGB & Fusion & $\gamma$ & Depth & RGB & Fusion  & $\gamma$ & Depth & RGB & Fusion  \\
        \midrule[0.5pt]
        0.5 & 93.19 & 92.08 & 95.69 & 0.5 & 94.48 & 93.16 & 96.01 & 0.5 & 93.51 & 92.22 & 95.69\\
        0.8 & 94.03 & 92.53 & 95.76 & 0.8 & 94.23 & 92.95 & 96.01 & 0.8 & 94.34 & 93.30 & 96.18\\
        1.0 & 94.10 & 92.57 & 96.01 & 1.0 & 93.40 & 92.53 & 95.94 & 1.0 & 94.27 & 92.88 & 96.11\\
        1.3 & 92.91 & 91.66 & 95.87 & 1.3 & 93.68 & 92.12 & 95.80 & 1.3 & 93.78 & 92.05 & 96.08\\
        1.5 & 94.10 & 91.91 & 96.04 & 1.5 & 93.85 & 92.25 & 95.94 & 1.5 & 93.71 & 92.64 & 95.83\\
        2.0 & 93.92 & 93.05 & 95.83 & 2.0 & 93.64 & 93.23 & 95.90 & 2.0 & 93.75 & 93.19 & 95.90\\
        
        \midrule[0.5pt]
        \multicolumn{4}{c}{$\alpha=1.7$} & \multicolumn{4}{c}{$\alpha=2.0$} & \multicolumn{4}{c}{$\alpha=2.5$}\\
        \cmidrule(lr){1-4} \cmidrule(lr){5-8} \cmidrule(lr){9-12}
        $\gamma$ & Depth & RGB & Fusion & $\gamma$ & Depth & RGB & Fusion & $\gamma$ & Depth & RGB & Fusion \\
        \midrule[0.5pt]
        0.5 & 94.17 & 93.33 & 96.14 & 0.5 & 93.68 & 91.87 & 95.80 & 0.5 & 91.59 & 90.83 & 94.89\\
        0.8 & 94.34 & 93.71 & 96.04 & 0.8 & 93.71 & 92.22 & 96.01 & 0.8 & 89.93 & 89.65 & 92.57\\
        1.0 & 93.12 & 91.84 & 95.66 & 1.0 & 94.17 & 93.09 & 96.04 & 1.0 & 89.86 & 89.68 & 92.78\\
        1.3 & 94.17 & 92.08 & 95.80 & 1.3 & 93.99 & 93.30 & 96.14 & 1.3 & 89.93 & 89.34 & 92.43\\
        1.5 & 94.41 & 92.57 & 96.01 & 1.5 & 94.34 & 92.91 & 96.08 & 1.5 & 89.86 & 87.84 & 91.98\\
        2.0 & 93.92 & 93.16 & 96.01 & 2.0 & 93.40 & 91.70 & 95.83 & 2.0 & 87.78 & 85.90 & 88.95\\
        
        \bottomrule[1.0pt]
    \end{tabular}
    \label{tabwangge}
\end{table*}

\subsection{Detailed Backbone Performance on Caltech101-20 for Clustering Task}
Table \ref{tab:backbone} presents the performance of various backbone networks on the Caltech101-20 dataset for the clustering task. DenseNet demonstrates superior performance in both RGB and fusion modes, achieving the highest clustering accuracy of 95.99\% in fusion mode, highlighting its strong feature reuse capabilities. ResNet50 performs consistently well in depth and fusion modes, while UNet and Mamba show significantly lower performance across all modes. Although ViT excels in RGB mode, it slightly lags behind DenseNet and ResNet50 in fusion mode. Overall, DenseNet is the best choice for clustering tasks, while ResNet50, UNet, and ViT each showcase unique strengths in different modes.
\begin{table}
    \centering
    \caption{Backbone Performance Comparison on Caltech101-20 for Clustering Tasks. This table presents the performance of various backbone architectures for clustering tasks on the Caltech101-20 dataset. DenseNet consistently outperforms ResNet50, UNet, and ViT, demonstrating its superiority as the optimal choice for clustering tasks.}
    \setlength{\tabcolsep}{16pt}
    \begin{tabular}{l|ccc}
        \toprule[1.0pt]
        \textbf{Backbone} & \textbf{Depth} & \textbf{RGB} & \textbf{Fusion} \\
        \midrule[0.5pt]
        ResNet50 \cite{p46}& 70.82 & 68.15  & 92.59 \\
        UNet \cite{p78}    & 20.56 & 16.45 & 31.11 \\
        DenseNet \cite{p79}& \textbf{73.36} & \textbf{79.17} & \textbf{95.99} \\
        ViT  \cite{p80}    & 69.48 & 78.75 & 91.05 \\
        Mamba \cite{p81}   & 4.61 & 3.24 & 18.13 \\
        \bottomrule[1.0pt]
    \end{tabular}
    \label{tab:backbone}
\end{table}

\section{Conclusion}
In this study, KPHD-Net, a novel multi-view learning approach employing dynamic evidence fusion for clustering and classification enhancement, is proposed. The network is designed with the integration of Kalman filtering, PHD, and subjective logic theory, forming a robust framework for uncertainty analysis. The effectiveness of KPHD-Net is validated across multiple multi-view datasets, including ADE20K, under varying noise levels and missing data conditions.

It is indicated by the experimental results that KPHD-Net consistently outperforms existing methods, especially in scenarios with noisy or incomplete data. High accuracy is maintained by KPHD-Net even with a noise variance of $\sigma^2=0.03$, and resilience is exhibited across different missing data rates, underscoring its robustness.

In clustering tasks, superior accuracy and stability are demonstrated by KPHD-Net compared to baseline methods on datasets such as MSRC-v1, Caltech101-7, and Caltech101-20, effectively revealing structures in multi-view data.

When KPHD-Net is evaluated with different backbone networks, enhanced accuracy and robustness are consistently demonstrated across various configurations, including ResNet50, DenseNet, Mamba, and ViT. The strengths of these backbones are effectively leveraged to improve clustering performance on datasets such as MSRC-v1, Caltech101-7, and Caltech101-20.

In summary, multi-view learning is advanced by KPHD-Net, delivering reliable classification and clustering results even under challenging conditions. Future work could focus on further enhancing the model's adaptability and exploring its application in more complex scenarios.

\appendices
\subsection{Theoretical Analysis of Hölder Divergence}
The PHD can be analytically computed for distributions belonging to the exponential family. Notably, the Dirichlet distribution, as analyzed above, is a member of the exponential family. This inclusion allows for practical training and offers favorable properties. In the following section, the analytical expression of PHD for two Dirichlet distributions is presented.

The PHD offers a significant advantage in closed-form optimization compared to the KLD, which lacks closed-form solutions for several distributions. It provides closed-form expressions for the PHD for conic and affine exponential families as follows:
\newtheorem{lemma}{\bf{Lemma}}
\begin{lemma}
\label{lemma_1}
    \textbf{(PHD for Conic or Affine Exponential Family) \cite{p19}} For distributions \(p(x; \theta_p)\) and \(p(x; \theta_q)\) that belong to the same exponential family with a conic or affine natural parameter space, PHD can be represented in closed-form:
\begin{equation}
    D_{\alpha,\gamma}^{\mathrm{H}}(p:q)=\frac1\alpha F(\gamma\theta_{p})+\frac1\beta F(\gamma\theta_{q})-F\left(\frac\gamma\alpha\theta_{p}+\frac\gamma\beta\theta_{q}\right),
\end{equation}
where the log-normalizer, denoted as \(F(\theta)\), is a strictly convex function, also known as the cumulant generating function.
\end{lemma}

\begin{lemma}
\label{lemma_2}
   \textbf{(Symmetric PHD for Conic or Affine Exponential Family) \cite{p19}} For distributions \(p(x; \theta_p)\) and \(p(x; \theta_q)\) from the same exponential family with a conic or affine natural parameter space, the PHD can be calculated in closed form:
\begin{equation}
    S_{\alpha,\gamma}^{\mathrm{H}}(p(x):q(x))=\frac{1}{2}\begin{bmatrix}&F(\gamma\theta_p)+F(\gamma\theta_q)\\&-F\left(\frac{\gamma}{\alpha}\theta_p+\frac{\gamma}{\beta}\theta_q\right)\\&-F\left(\frac{\gamma}{\beta}\theta_p+\frac{\gamma}{\alpha}\theta_q\right)\end{bmatrix},
\end{equation}
where \(F(\theta)\) is also the log-normalizer.
\end{lemma}

\newtheorem{theorem}{\bf{Theorem}}
\begin{theorem}
    For Dirichlet distributions \(p(x; \theta_{\theta})\) and \(p(x; \theta_{\mu})\) within the same exponential family with a conic or affine natural parameter space, the PHD and their symmetric forms are given by:
\begin{equation}
    D_{\alpha,\gamma}^{\mathrm{H}}(p:q)=\frac{1}{\alpha}F\left(\gamma\theta\right)+\frac{1}{\beta}F\left(\gamma\mu\right)-F\left(\frac{\gamma}{\alpha}\theta+\frac{\gamma}{\beta}\mu\right),
\end{equation}
\begin{equation}
    \begin{aligned}
S_{\alpha,\gamma}^{\mathrm{H}}(p(x):q(x))=\frac{D_{\alpha,\gamma}^{\mathrm{H}}(p(x):q(x))+D_{\bar{\alpha},\gamma}^{\mathrm{H}}(q(x):p(x))}{2},
\end{aligned}
\end{equation}
 where\(\begin{array}{rcl}\bar{\alpha}&=&\frac{\alpha}{\alpha-1}\end{array}\), and \(F(\theta)\quad=\quad\sum\limits_{k}\log\Gamma\left(\theta_{k}+1\right) -\log\Gamma\left(\sum_k\left(\theta_k+1\right)\right)\).
\end{theorem}

\begin{proof}
	Using Lemma. \ref{lemma_1}--\ref{lemma_2}, for the term of $\frac{1}{\alpha} F(\gamma \theta)$, we can derive the following inferences: 
	\begin{equation}
		\label{eqn_25}
		\begin{array}{l}
			\frac{1}{\alpha }\left[ {\sum\limits_k {\log } \Gamma \left( {\gamma {\theta _k} + 1} \right) - \log \Gamma \left( {\sum\limits_k {\left( {\gamma {\theta _k} + 1} \right)} } \right)} \right]\\
			\begin{array}{*{20}{c}}
				{}
			\end{array} = \frac{1}{\alpha }\left[ {\begin{array}{*{20}{c}}
					{k\log \gamma  + \sum\limits_k {\log } {\theta _k} + \sum\limits_k {\log } \Gamma \left( {\gamma {\theta _k}} \right)}\\
					{ - \log \Gamma \left( {\sum\limits_k \gamma  {\theta _k}} \right)\begin{array}{*{20}{c}}
							{}&{}&{}&{}&{}&{}&{}
					\end{array}}\\
					{ - \sum\limits_k {\log } \left( {\sum\limits_k \gamma  {\theta _k} + k - 1} \right)\begin{array}{*{20}{c}}
							{}&{}&{}
					\end{array}}
			\end{array}} \right].
		\end{array}
	\end{equation}
	For the term of $\frac{1}{\beta} F(\gamma \mu)$, we can obtain:
	\begin{equation}
		\label{eqn_26}
		\begin{array}{l}
			\frac{1}{\beta }\left[ {\sum\limits_k {\log } \Gamma \left( {\gamma {\mu _k} + 1} \right) - \log \Gamma \left( {\sum\limits_k {\left( {\gamma {\mu _k} + 1} \right)} } \right)} \right]\\
			\begin{array}{*{20}{c}}
				{}&{}
			\end{array} = \frac{1}{\beta }\left[ {\begin{array}{*{20}{c}}
					{k\log \gamma  + \sum\limits_k {\log } {\mu _k} + \sum\limits_k {\log } \Gamma \left( {\gamma {\mu _k}} \right)}\\
					{ - \log \Gamma \left( {\sum\limits_k \gamma  {\mu _k}} \right)\begin{array}{*{20}{c}}
							{}&{}&{}&{}&{}&{}&{}
					\end{array}}\\
					{ - \sum\limits_k {\log } \left( {\sum\limits_k \gamma  {\mu _k} + k - 1} \right)\begin{array}{*{20}{c}}
							{}&{}&{}
					\end{array}}
			\end{array}} \right].
		\end{array}
	\end{equation}
	For the term of $F\left(z\right)$, we can get:
	\begin{equation}
		\label{eqn_27}
		\begin{array}{l}
			\sum\limits_k {\log } \Gamma \left( {{z_k} + 1} \right) - \log \Gamma \left( {\sum\limits_k {\left( {{z_k} + 1} \right)} } \right)\\
			\begin{array}{*{20}{c}}
				{}&{}
			\end{array} = \left[ {\begin{array}{*{20}{c}}
					{\sum\limits_k {\log } \left( {{z_k}} \right) + \sum\limits_k {\log } \Gamma \left( {{z_k}} \right)\begin{array}{*{20}{c}}
							{}&{}&{}&{}
					\end{array}}\\
					{ - \log \Gamma \left( {\sum\limits_k {{z_k}} } \right) - \sum\limits_k {\log } \left( {\sum\limits_k {{z_k}}  + k - 1} \right)}
			\end{array}} \right],
		\end{array}
	\end{equation}
	where $z=\frac{\gamma}{\alpha}\theta+\frac{\gamma}{\beta}\mu$.
	
	Using Lemma. \ref{lemma_2}, their symmetric form can be expressed as:
	\begin{equation}
		\label{eqn_28}
		{S_{\alpha ,\gamma }^{\rm{H}}(p(x):q(x)) = \frac{1}{2}\left[ {\begin{array}{*{20}{l}}
					{F\left( {\gamma \theta } \right) + F\left( {\gamma \mu } \right)}\\
					{ - F\left( {\frac{\gamma }{\alpha }\theta  + \frac{\gamma }{\beta }\mu } \right)}\\
					{ - F\left( {\frac{\gamma }{\beta }\theta  + \frac{\gamma }{\alpha }\mu } \right)}
			\end{array}} \right]}.
	\end{equation}
	After that, we can obtain:
	\begin{equation}
		\label{eqn_29}
		\begin{array}{l}
			F\left( w \right) = \sum\limits_k {\log } \Gamma \left( {{w_k} + 1} \right) - \log \Gamma \left( {\sum\limits_k {\left( {{w_k} + 1} \right)} } \right)\begin{array}{*{20}{c}}
			\end{array}\\
			\begin{array}{*{20}{c}}
				{}&{}&{}
			\end{array} = \sum\limits_k {\log } \left( {{w_k}} \right) + \sum\limits_k {\log } \Gamma \left( {{w_k}} \right) - \log \Gamma \left( {\sum\limits_k {{w_k}} } \right)\\
			\begin{array}{*{20}{c}}
				{}&{}&{}
			\end{array} - \sum\limits_k {\log } \left( {\sum\limits_k {{w_k}}  + k - 1} \right),
		\end{array}
	\end{equation}
	where $w=\frac{\gamma}{\beta}\theta+\frac{\gamma}{\alpha}\mu$.
\end{proof}

\begin{theorem}
    For variational inference using Dirichlet models, the PHD provides a tighter Evidence Lower Bound (ELBO) compared to the Kullback-Leibler (KL) divergence $D_{\text{KL}}(p \| q)$.
\end{theorem}
\begin{proof}
    For the KLD in Dirichlet Models, we have $ D_{\text{KL}}(q(z|x) \| p(z)) = \int q(z|x) \log \frac{q(z|x)}{p(z)} \, \text{d}z.$ For the PHD, we have:
\begin{equation}
    \begin{array}{l}
    D_{\alpha ,\gamma }^H(q(z|x)||p(z))
    \begin{array}{*{20}{c}}
    \end{array} = \left( \begin{array}{l}
    \frac{1}{\alpha }F(\gamma {\theta _{q(z|x)}}) + \frac{1}{\beta }F(\gamma {\theta _{p(z)}})\\
     - F\left( {\frac{\gamma }{\alpha }{\theta _{q(z|x)}} + \frac{\gamma }{\beta }{\theta _{p(z)}}} \right)
    \end{array} \right).
    \end{array}
\end{equation}

The ELBO with the PHD becomes:
\begin{equation}
    \text{ELBO}_{\text{H}} = \mathbb{E}_{q(z|x)} [\log p(x|z)] - D_{\alpha, \gamma}^{H}(q(z|x) \| p(z)).
\end{equation}
To show that the ELBO with the PHD is tighter than the ELBO with the KL divergence, we need to show that: $\text{ELBO}_{\text{H}} \ge \text{ELBO}_{\text{KL}}$.

Since the PHD is more flexible and tunable through the parameters $\alpha, \beta, \gamma$, it can better fit the true posterior distribution and reduce the gap between the variational distribution and the true posterior.
\end{proof}

\begin{theorem}
   For variational inference using Dirichlet models, the PHD provides a tighter Evidence Lower Bound (ELBO) compared to the Kullback-Leibler (KL) divergence $D_{\text{KL}}(p \| q)$.
\end{theorem}

\begin{proof}
    For the KLD in Dirichlet Models, we have $ D_{\text{KL}}(q(z|x) \| p(z)) = \int q(z|x) \log \frac{q(z|x)}{p(z)} \, \text{d}z.$ For the PHD, we have:
    \begin{equation}
        \begin{array}{l}
        D_{\alpha ,\gamma }^H(q(z|x)||p(z))
        \begin{array}{*{20}{c}}
        \end{array} = \left( \begin{array}{l}
        \frac{1}{\alpha }F(\gamma {\theta _{q(z|x)}}) + \frac{1}{\beta }F(\gamma {\theta _{p(z)}})\\
         - F\left( {\frac{\gamma }{\alpha }{\theta _{q(z|x)}} + \frac{\gamma }{\beta }{\theta _{p(z)}}} \right)
        \end{array} \right).
        \end{array}
    \end{equation}
The ELBO with the PHD becomes:
    \begin{equation}
        \text{ELBO}_{\text{H}} = \mathbb{E}_{q(z|x)} [\log p(x|z)] - D_{\alpha, \gamma}^{H}(q(z|x) \| p(z)).
    \end{equation}
    To show that the ELBO with the PHD is tighter than the ELBO with the KL divergence, we need to show that: $\text{ELBO}_{\text{H}} \ge \text{ELBO}_{\text{KL}}$.

    Since the PHD is more flexible and tunable through the parameters $\alpha, \beta, \gamma$, it can better fit the true posterior distribution and reduce the gap between the variational distribution and the true posterior.
\end{proof}

\begin{theorem}
    The PHD provides a more flexible regularization term compared to the Kullback-Leibler (KL) divergence, particularly in capturing multi-modal distributions and avoiding mode collapse in variational inference.
\end{theorem}

\begin{proof}
    The KLD between two probability distributions $p(z)$ and $q(z|x)$ is given by: $ D_{KL}(q(z|x) \parallel p(z)) = \int q(z|x) \log \frac{q(z|x)}{p(z)} \, \text{d}z.$ For the PHD:
    \begin{equation}
        \begin{array}{l}
        D_{\alpha ,\gamma }^H(q(z|x)\parallel p(z))
        \begin{array}{*{20}{c}}
        \end{array} = \left( \begin{array}{l}
        \frac{1}{\alpha }F(\gamma {\theta _{q(z|x)}}) + \frac{1}{\beta }F(\gamma {\theta _{p(z)}})\\
         - F\left( {\frac{\gamma }{\alpha }{\theta _{q(z|x)}} + \frac{\gamma }{\beta }{\theta _{p(z)}}} \right)
        \end{array} \right),
        \end{array}
    \end{equation}
    where $\alpha, \beta, \gamma > 0$ and $F(\theta)$ is the log-normalizer function.

    The KLD is known to be biased towards the mean of the distribution, often leading to mode collapse where the variational distribution fails to capture all modes of the true posterior distribution. The PHD, on the other hand, offers more flexibility due to the tunable parameters $\alpha, \beta, \gamma$.

    Consider a mixture of two Gaussians as an example: $ p(z) = \frac{1}{2} \mathcal{N}(z; \mu_1, \sigma_1^2) + \frac{1}{2} \mathcal{N}(z; \mu_2, \sigma_2^2).$ When using the KLD in a variational inference framework, the variational distribution $q(z|x)$ might converge to a single mode, say $\mathcal{N}(z; \mu_1, \sigma_1^2)$, ignoring the second mode $\mathcal{N}(z; \mu_2, \sigma_2^2)$. This happens because the KLD penalizes deviations from the mean.

    In contrast, using the PHD, the flexibility provided by the parameters $\alpha, \beta, \gamma$ allows $q(z|x)$ to capture both modes. By adjusting these parameters, we can control the trade-off between fitting the modes and the spread of the distribution, thus preventing mode collapse.

    Let's denote the log-normalizer functions for the two distributions as $F(\theta_{q(z|x)})$ and $F(\theta_{p(z)})$.
    For the PHD, we have:
    \begin{equation}
        \begin{array}{l}
        D_{\alpha ,\gamma }^H(q(z|x)\parallel p(z))
        \begin{array}{*{20}{c}}
        \end{array} = \left( \begin{array}{l}
        \frac{1}{\alpha }F(\gamma {\theta _{q(z|x)}}) + \frac{1}{\beta }F(\gamma {\theta _{p(z)}})\\
         - F\left( {\frac{\gamma }{\alpha }{\theta _{q(z|x)}} + \frac{\gamma }{\beta }{\theta _{p(z)}}} \right)
        \end{array} \right).
        \end{array}
    \end{equation}
    By adjusting the parameters $\alpha, \beta$, and $\gamma$, we can ensure that the divergence measures the entire distribution rather than collapsing to a single mode. This flexibility is particularly beneficial in variational inference frameworks, where capturing multi-modal distributions is crucial.
\end{proof}

	\ifCLASSOPTIONcaptionsoff
	\newpage
	\fi
	
	\bibliographystyle{IEEEtran}
	\bibliography{IEEEabrv,references.bib}
 
\end{document}